\PassOptionsToPackage{ruled,vlined,algo2e}{algorithm2e}
\documentclass{article}
\usepackage[margin=1in]{geometry}
\usepackage[utf8]{inputenc}
\usepackage[T1]{fontenc}
\usepackage{microtype}
\usepackage{amsmath,amssymb,amsthm}
\usepackage{mathtools}
\usepackage{bm}
\usepackage{booktabs}
\usepackage{multirow}
\usepackage{graphicx}
\usepackage{xcolor}
\usepackage{framed}
\usepackage{enumitem}
\usepackage{tikz}
\usetikzlibrary{patterns}
\usepackage{pgfplots}
\pgfplotsset{compat=1.18}
\usepackage{algorithm2e}
\usepackage[round]{natbib}
\usepackage[pagebackref,colorlinks=true,linkcolor=blue,citecolor=blue,urlcolor=blue]{hyperref}
\usepackage[capitalize,noabbrev]{cleveref}
\hypersetup{pdftitle={Asymmetric Likelihood-Free Hypothesis Testing},pdfauthor={Gunjan Kumar, Yash Pote, Jonathan Scarlett}}
\usepackage{thmtools}
\usepackage{thm-restate}
\renewcommand{\geq}{\geqslant}
\renewcommand{\leq}{\leqslant}
\setlist[itemize]{topsep=2pt, partopsep=0pt}

\SetAlgoVlined
\SetAlCapHSkip{0pt}
\DontPrintSemicolon
\SetKwFor{For}{for}{do}{}
\SetKwFor{While}{while}{do}{}
\SetKwFor{ForEach}{foreach}{do}{}

\let\cite\citep
\renewcommand*{\backrefalt}[4]{\ifcase #1 \or $\uparrow$\,#2\else $\uparrow$\,#2\fi}

\theoremstyle{plain}
\newtheorem{theorem}{Theorem}
\newtheorem{lemma}[theorem]{Lemma}
\newtheorem{proposition}[theorem]{Proposition}

\theoremstyle{definition}

\newtheorem{remark}[theorem]{Remark}

\crefname{lemma}{Lemma}{Lemmas}
\Crefname{lemma}{Lemma}{Lemmas}
\crefname{corollary}{Corollary}{Corollaries}
\Crefname{corollary}{Corollary}{Corollaries}
\crefname{definition}{Definition}{Definitions}
\Crefname{definition}{Definition}{Definitions}
\crefname{remark}{Remark}{Remarks}
\Crefname{remark}{Remark}{Remarks}
\crefname{fact}{Fact}{Facts}
\Crefname{fact}{Fact}{Facts}
\crefname{proposition}{Proposition}{Propositions}
\Crefname{proposition}{Proposition}{Propositions}

\newcommand{\tv}{\mathrm{d_{TV}}}
\newcommand{\UST}{\mathrm{UST}}

\newcommand{\LF}{\mathsf{LF}}
\newcommand{\RLF}{R_{\LF}}
\providecommand{\E}{\mathbb{E}}
\providecommand{\Var}{\operatorname{Var}}

\newcommand{\Poi}{\operatorname{Poi}}
\newcommand{\poi}{\operatorname{poi}}

\title{Asymmetric Likelihood-Free Hypothesis Testing}

\author{
Gunjan Kumar\\
IIT Kanpur\\
\texttt{gunjan@cse.iitk.ac.in}
\and
Yash Pote\\
National University of Singapore\\
\texttt{yashppote@gmail.com}
\and
Jonathan Scarlett\\
National University of Singapore\\
\texttt{scarlett@comp.nus.edu.sg}
}
\date{}

\begin{document}
\maketitle
\begin{abstract}
In likelihood-free hypothesis testing (LFHT), given samples from two unknown distributions $P,Q$ on $[n]$ with $\tv(P,Q)\geq\varepsilon$, and from a third distribution $D\in\{P,Q\}$, the goal is to identify whether $D=P$ or $D=Q$. We study asymmetric likelihood-free hypothesis testing, where the sample budgets $s_P$ and $s_Q$ for $P$ and $Q$ may differ, and characterize the feasible triples $(s_P,s_Q,s_D)$ up to logarithmic factors. We establish that unequal reference budgets can give a polynomial saving. For instance, for constant separation, $n$ samples from $P$ and $n^{3/4}$ from $Q$ allow testing with $\widetilde O(n^{1/4})$ samples from $D$, whereas both existing baselines---symmetric LFHT and unequal-sample closeness testing---require $\Omega(\sqrt n)$ samples from $D$.

We further apply our results to non-adaptive distribution 2-clustering, where the sample allocation among $k$ distributions must be fixed before observing any data. We obtain matching upper and lower bounds for this problem up to logarithmic factors in terms of the domain size, number of distributions, smaller cluster size, and separation, both with one cluster distribution known and with both unknown. Comparing with known bounds for adaptive sample allocation gives a worst-case adaptivity gap of $\widetilde\Theta(k^{1/4})$.
\end{abstract}

\section{Introduction}\label{sec:intro}

Likelihood-free inference arises when data can be simulated from a model, but its likelihood cannot be evaluated. In likelihood-free hypothesis testing (LFHT)~\citep{GP24,GHP23}, we observe samples from two unknown reference distributions $P,Q$ on $[n]=\{1,\ldots,n\}$ at total variation distance $\tv(P,Q)\geq\varepsilon$, and from a target distribution $D\in\{P,Q\}$. The goal is to identify $D$ with probability at least $2/3$. We write $s_P,s_Q,s_D$ for the three sample budgets.

While prior work gives tight bounds in the symmetric case $s_P=s_Q$~\citep{GP24,GHP23}, the asymmetric case $s_P\ne s_Q$ is also of significant interest. For instance, an existing model may have many simulated samples available, while a new model has only a few. We ask how the extra samples from the first model reduce the number needed from the target distribution.

In this paper, we characterize the feasible triples $(s_P,s_Q,s_D)$ up to logarithmic factors. To exemplify the benefit of considering unequal samples, suppose that the separation $\varepsilon$ is constant, and consider $s_P=n$ and $s_Q=n^{3/4}$. Discarding excess $P$-samples to use symmetric LFHT requires $s_D=\Omega(\sqrt n)$. Ignoring $Q$ and testing closeness between $P$ and $D$ also requires $s_D=\Omega(\sqrt n)$. In contrast, using both reference batches, our bound gives $s_D=\widetilde\Theta(n^{1/4})$ (\cref{sec:main}), thus providing a polynomial saving. The characterization also extends to more than two hypotheses (\cref{sec:d-ary}). Here and below, $\widetilde O$ and $\widetilde\Theta$ hide polylogarithmic factors in $n$ and, for clustering, the number of distributions $k$.

\subsection{Application to Non-Adaptive Clustering} \label{sec:intro_clustering}

In distribution 2-clustering, each of $k$ input distributions equals one of two distributions $P,Q$ with $\tv(P,Q)\geq\varepsilon$, and the goal is to recover their partition using as few total samples as possible. Prior work~\citep{KPS26} allows adaptive sampling, i.e., earlier sample outcomes can be used to decide which distribution to sample from next. We determine the sample complexity under non-adaptive sampling, i.e., the allocation must be fixed before observing any samples. If the smaller cluster has known size $r$, and both cluster distributions are unknown, the optimal total budget is
\begin{align*}
T_{\mathrm{na}}(n,k,r,\varepsilon)
=\widetilde\Theta\!\left(
\frac{k^{2/3}n^{2/3}}{r^{1/3}\varepsilon^{4/3}}
+\frac{k\sqrt n}{\sqrt r\,\varepsilon^2}
+\frac{k}{\varepsilon^2}\right).
\end{align*}
We also obtain matching upper and lower bounds up to logarithmic factors when one cluster distribution is known (\cref{sec:clustering}). Comparing with the adaptive bounds of~\citet{KPS26} gives a worst-case adaptivity gap of $\widetilde\Theta(k^{1/4})$, resolving their open question. For example, when $n=k$, $r\asymp\sqrt k$, and $\varepsilon$ is constant, the adaptive and non-adaptive budgets are $\widetilde\Theta(k)$ and $\widetilde\Theta(k^{5/4})$, respectively.

The connection to LFHT comes from the observation that, once an exemplar from each cluster is identified, classifying the remaining distributions reduces to asymmetric LFHT, since the two exemplars may come with unequal sample budgets.

\section{Preliminaries and Background}\label{sec:setup}

For the problem setup in Section \ref{sec:intro}, let $\RLF(\varepsilon)$ be the set of nonnegative integer triples $(s_P,s_Q,s_D)$ for which a test using these sample budgets has error at most $1/3$ for every pair $P,Q$ with $\tv(P,Q)\geq\varepsilon$ and either choice $D\in\{P,Q\}$. Assume without loss of generality that $s_P\geq s_Q$, and set
\begin{align}
    \alpha:=\max\left\{1,\frac{n}{\max\{s_P,s_D\}}\right\}. \label{eq:alpha}
\end{align}
The factor $\alpha$ equals one once either $P$ or $D$ has at least $n$ samples. 

For a distribution $P$ on $[n]$, let $\|P\|_j^j=\sum_i P(i)^j$ and $\|P\|_\infty=\max_i P(i)$. For distributions $P,Q$, let $\chi^2(P\|Q)=\sum_x\frac{(P(x)-Q(x))^2}{Q(x)}$.  We will use the inequality $\tv(P,Q)\leq\frac{1}{2}\sqrt{\chi^2(P\|Q)}$, which follows from Cauchy--Schwarz. We write $\Poi(\lambda)$ for the Poisson distribution with mean $\lambda\geq0$ and $\poi_\lambda(m)=e^{-\lambda}\lambda^m/m!$ for its probability at a nonnegative integer $m$.

The notation $\asymp_{\log n}$ hides constant and polylogarithmic factors in $n$ in the sample budgets, and the symbols $\lesssim$, $\gtrsim$, and $\asymp$ hide constant factors that are independent of all problem parameters.  The hidden terms are often required to be sufficiently large without explicitly stating so.

\paragraph{Symmetric LFHT.}
Let $\RLF^{\mathrm{sym}}(\varepsilon)$ be the restriction of
$\RLF(\varepsilon)$ to $s_P=s_Q$.
\citet[Theorem~2]{GP24} show that
\begin{align}\label{eq:GP-sym}
	\RLF^{\mathrm{sym}}(\varepsilon)
	\asymp_{\log n}
	\left\{
	(s_P,s_D):
	s_D \geq \frac{1}{\varepsilon^2},\quad
	s_P \geq \frac{\sqrt{n\alpha}}{\varepsilon^2},\quad
	s_Ps_D \geq \frac{n\alpha}{\varepsilon^4}
	\right\}.
\end{align}
The three conditions roughly correspond to (i)  enough target samples for testing even if $P$ and $Q$ were known, (ii) enough reference samples such that $P$ and $Q$ are not too poorly sampled, and (iii) sufficient ``joint information'' through the product of sample counts.

\subsection{Related Work}

\paragraph{Distribution testing.}
Distribution testing refers to the task of testing whether a distribution has a given property using samples \cite{C22}. In \textit{identity testing}, one must distinguish between $\tv(P,Q)=0$ and $\tv(P,Q)\geq\varepsilon$ for a known distribution $P$ and an unknown distribution $Q$. Uniformity testing is a special case where $P$ is uniform. If both distributions are unknown, the problem is called \textit{closeness testing} \citep{BFRS00,CDVV14}. Our test uses ingredients similar to those of these closeness testers, such as estimators of squared $\ell_2$ distances and sample-based flattening of \citet{DK16}.

\paragraph{LFHT.}
In information theory, \textit{classification with empirically observed statistics} is a generalization of the Neyman--Pearson framework to situations where the candidate hypotheses are completely unknown but we instead have a set of training samples from them. This line of work was initiated by \citet{Z02} and \citet{G02}. Their work considered a fixed alphabet and studied the error exponents in terms of the size of the training and test sample sets. \citet{KWTV13} and \citet{HM12} let the alphabet grow with the sample size, under equal training lengths and the assumption that no symbol has probability much larger than under the uniform distribution, and ask when consistent classification is possible at all. \citet{GP24} revisited the problem under the name likelihood-free hypothesis testing (LFHT), and shifted the focus from asymptotic exponents to finite-sample guarantees. For a fixed error probability and equal reference budgets, they characterize the sample complexity in terms of the domain size and the TV separation, for several classes of distributions including the discrete distributions considered here. \citet{GHP23} show that tests based on classification achieve near-optimal sample complexity, including the dependence on the error probability, and give the lower-bound construction that we adapt in \cref{sec:lb}. These works already trade off the target budget $s_D$ against a common reference budget. We determine the tradeoff when the two reference budgets $s_P$ and $s_Q$ themselves differ.

\paragraph{Clustering of distributions.}
\citet{KPS26} study the clustering problem  from Section \ref{sec:intro_clustering} with \emph{adaptive strategies} and obtain matching upper and lower bounds up to logarithmic factors. Their algorithms first find an exemplar from each cluster and then classify the remaining distributions. In the non-adaptive setting, the allocation must be fixed before observing any samples, so the budgets for classification cannot depend on which exemplars are found. \citet{KPS26} ask whether this constraint makes non-adaptive sampling require strictly more samples. We answer this positively with a worst-case gap of $\widetilde\Theta(k^{1/4})$ where $k$ is the number of distributions.  Clustering from bandit feedback (with adaptive sampling) was also studied by \citet{YHTS25}, but it is less directly related since it focuses on the sample complexity as the error probability tends to zero with the instance fixed, leaving the dependence on $n$, $k$, and $r$ in lower-order terms.

\section{Main Result}\label{sec:main}
Our main result characterizes the feasible sample budgets for asymmetric LFHT; see also \cref{sec:d-ary} for the extension to more than two hypotheses.
\begin{theorem}\label{thm:char}
	For $n\geq 10$ and $\varepsilon\in(0,1/8]$, under the convention $s_P \geq s_Q$, we have
	\begin{align}\label{eq:char}
		\RLF(\varepsilon)
		\asymp_{\log n}
		\left\{
		(s_P,s_Q,s_D):
		s_D \geq \frac{1}{\varepsilon^2},\quad
		s_P \geq \frac{\sqrt{n\alpha}}{\varepsilon^2},\quad
		s_D(s_Q+s_D) \geq \frac{n\alpha}{\varepsilon^4}
		\right\},
	\end{align}
    where $\alpha$ is defined in \eqref{eq:alpha}.
\end{theorem}
The three conditions share a similar interpretation to the symmetric case.  The first condition is necessary even when both reference distributions are known. The second establishes that we require sufficiently many samples from at least one reference distribution. The third allows additional $D$-samples to compensate for fewer $Q$-samples.

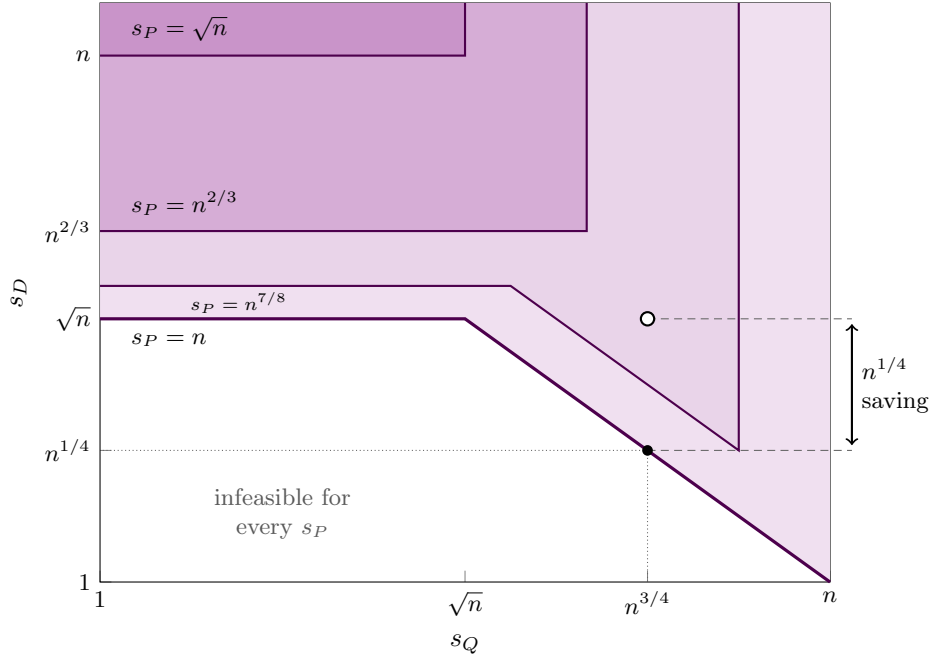
\begin{figure}[t!]
    \centering
    \begin{tikzpicture}
        \begin{axis}[
            width=0.68\linewidth, height=0.56\linewidth,
            xmin=0, xmax=1, ymin=0, ymax=1.1,
            xlabel={$s_Q$}, ylabel={$s_D$},
            xtick={0,0.5,0.75,1},
            xticklabels={$1$,$\sqrt n$,$n^{3/4}$,$n$},
            ytick={0,0.25,0.5,0.6667,1},
            yticklabels={$1$,$n^{1/4}$,$\sqrt n$,$n^{2/3}$,$n$},
            ticklabel style={font=\small},
            axis lines=box, xtick pos=bottom, ytick pos=left,
            clip=false,
        ]
            \fill[violet!12] (axis cs:0,1.1) -- (axis cs:0,0.5)
                -- (axis cs:0.5,0.5) -- (axis cs:1,0) -- (axis cs:1,1.1) -- cycle;
            \fill[violet!17] (axis cs:0,1.1) -- (axis cs:0,0.5625)
                -- (axis cs:0.5625,0.5625) -- (axis cs:0.875,0.25)
                -- (axis cs:0.875,1.1) -- cycle;
            \fill[violet!32] (axis cs:0,1.1) -- (axis cs:0,0.6667)
                -- (axis cs:0.6667,0.6667) -- (axis cs:0.6667,1.1) -- cycle;
            \fill[violet!42] (axis cs:0,1.1) -- (axis cs:0,1)
                -- (axis cs:0.5,1) -- (axis cs:0.5,1.1) -- cycle;

            \draw[thick, violet!60!black] (axis cs:0,1)
                -- (axis cs:0.5,1) -- (axis cs:0.5,1.1);
            \draw[thick, violet!60!black] (axis cs:0,0.6667)
                -- (axis cs:0.6667,0.6667) -- (axis cs:0.6667,1.1);
            \draw[thick, violet!60!black] (axis cs:0,0.5625)
                -- (axis cs:0.5625,0.5625) -- (axis cs:0.875,0.25)
                -- (axis cs:0.875,1.1);
            \draw[very thick, violet!60!black] (axis cs:0,0.5)
                -- (axis cs:0.5,0.5) -- (axis cs:1,0);

            \node[font=\small, anchor=south west] at (axis cs:0.03,1.01)
                {$s_P=\sqrt n$};
            \node[font=\small, anchor=south west] at (axis cs:0.03,0.68)
                {$s_P=n^{2/3}$};
            \node[font=\scriptsize, inner sep=0pt] at (axis cs:0.19,0.53125)
                {$s_P=n^{7/8}$};
            \node[font=\small, anchor=north west] at (axis cs:0.03,0.49)
                {$s_P=n$};

            \node[gray!70!black, font=\small, align=center]
                at (axis cs:0.25,0.13) {infeasible for\\every $s_P$};
            \draw[densely dotted, black!60] (axis cs:0.75,0.25) -- (axis cs:0.75,0);
            \draw[densely dotted, black!60] (axis cs:0.75,0.25) -- (axis cs:0,0.25);
            \draw[densely dashed, black!60] (axis cs:0.75,0.5) -- (axis cs:1.03,0.5);
            \draw[densely dashed, black!60] (axis cs:0.75,0.25) -- (axis cs:1.03,0.25);
            \draw[<->, thick] (axis cs:1.03,0.49) -- (axis cs:1.03,0.26)
                node[midway, right, font=\small, align=left] {$n^{1/4}$\\saving};
            \draw[black, fill=white, thick] (axis cs:0.75,0.5) circle (2.5pt);
            \fill[black] (axis cs:0.75,0.25) circle (2pt);
        \end{axis}
    \end{tikzpicture}
    \caption{Slices of the feasible region of \cref{thm:char} at four values
    of $s_P$, with $\varepsilon$ constant. Constant and logarithmic factors
    are suppressed, and both axes are logarithmic.}
    \label{fig:region}
\end{figure}

The feasible region is a set in three dimensions. \Cref{fig:region} shows four slices of it at fixed values of $s_P$, with $\varepsilon$ constant. The lower boundary of each slice has up to three parts. On the horizontal part, the required $s_D$ does not depend on $s_Q$. On the diagonal part, $s_Qs_D\asymp n\alpha$, so additional $Q$-samples lower the required $s_D$. The vertical part comes from the convention $s_P \geq s_Q$. As $s_P$ grows, the lower boundary moves down until $s_P=n$, and stays the same for larger $s_P$, so additional $P$-samples no longer help. The diagonal part appears only when $s_P>n^{2/3}$. For smaller $s_P$, the horizontal part already lies at $s_D\geq s_P$, and additional $Q$-samples do not reduce the required $s_D$.

The dot in \cref{fig:region} marks the example from the introduction, with $\varepsilon$ constant, $s_P=n$, and $s_Q=n^{3/4}$. Here $\alpha=1$, and the third condition $s_D(n^{3/4}+s_D)\gtrsim n$ gives $s_D=\widetilde\Theta(n^{1/4})$. Both baselines need $s_D\gtrsim\sqrt n$, shown by the open circle.

\subsection{Proof Ideas}

We summarize some of the main ideas here, and defer the details to Sections \ref{sec:ub} and \ref{sec:lb}.

Our upper bound adapts the quadratic statistic used for symmetric LFHT~\citep{GP24}. First, the distributions $P,Q$ are transformed into flattened distributions $\tilde p,\tilde q$ on a common expanded domain, preserving their TV distance while ensuring $\|\tilde p\|_\infty=O( (\log n)/s_P)$ and $\|\tilde q\|_\infty=O((\log n)/s_Q)$ with high probability.
After flattening, our algorithm estimates the difference between the squared $\ell_2$ distances from $\tilde d$ to $\tilde p$ and $\tilde q$. Having more samples from $P$ helps control not only the variance terms involving $P$, but also those involving $Q$. Specifically, the inequality
$$
\|\tilde q\|_2^2\leq 2\|\tilde p\|_2^2+2\|\tilde p-\tilde q\|_2^2
$$
allows us to control the terms involving $Q$ in the variance of our test statistic using the better-sampled reference $P$ and the squared separation between $P$ and $Q$.  Although larger separation may increase the variance, it also provides a distinguishing signal, which grows faster than the standard deviation.

The lower bound adapts the heavy/light construction of~\citet{GHP23}. On light pairs of domain elements, $P$ has equal $O(1/n)$ mass while $Q$ has a random hidden deviation from uniformity.   A single sample from such a pair is uninformative for distinguishing these two cases, so \emph{collisions} are required in the sense of having either two $D$-samples or one $D$-sample and one $Q$-sample from the same pair. The collision requirements lead to an $s_D(s_Q+s_D)$ factor in the lower bound.  On heavy pairs, all of $P,Q,D$ coincide and have equal $O(1/s_P)$ mass; these pairs hide which collisions are informative, yielding an additional dependence on $s_P$ when it is below the domain size $n$. %

\section{Upper Bound}\label{sec:ub}

In this section, we prove the upper bound in \cref{thm:char}. Recall the assumption $s_P\geq s_Q$, and let $(s_P,s_Q,s_D)$ satisfy the three conditions in~\eqref{eq:char}, allowing polylogarithmic factors in $n$ as indicated.  We construct a test that distinguishes $D=P$ from $D=Q$ with error probability at most $\frac13$.

We specifically consider a triplet on the right-hand side of~\eqref{eq:char} and multiply all three budgets by $C \log n$ for sufficiently large $C$, then relabel the resulting budgets as $(s_P,s_Q,s_D)$.  This can only decrease $\alpha$, and the condition on $s_D(s_Q+s_D)$ in~\eqref{eq:char} gains an additional factor of $(\log n)^2$, as required by the variance analysis in \cref{sec:var}. We  then proceed via three cases:
\begin{enumerate}
    \item If $s_D\gtrsim \frac{\sqrt{n\alpha}}{\varepsilon^2}$, run the closeness tester of~\citet{DK16} on $P$ and $D$.
    \item Otherwise, if $s_Q>n$, run symmetric LFHT using $s_Q$ samples from each reference distribution.
    \item Otherwise, run \cref{alg:lfht-case1} below with the number of $P$-samples truncated to $n$ (discarding any beyond that if $s_P > n$).
\end{enumerate}
We proceed to show that the test succeeds in each case.

\paragraph{Case 1:} In this case, $s_D\gtrsim \frac{\sqrt{n\alpha}}{\varepsilon^2}$. By~\eqref{eq:char}, the same lower bound holds for $s_P$. We apply the closeness tester of~\citet[Proposition~2.11]{DK16} to $P$ and $D$. This tester succeeds provided
\begin{align}\label{eq:baselineB}
\min(s_P,s_D)
&\gtrsim \frac{\sqrt n}{\varepsilon^2},
\qquad\text{and}\qquad
\min(s_P,s_D)^2\max(s_P,s_D)
\gtrsim \frac{n^2}{\varepsilon^4}.
\end{align}
The first condition follows immediately from \eqref{eq:char} and $\alpha\geq1$. For the second condition, the case condition and~\eqref{eq:char} give $\min(s_P,s_D)^2\gtrsim \frac{n\alpha}{\varepsilon^4}$. Multiplying by $\max(s_P,s_D)$ and using $\alpha\max(s_P,s_D)\geq n$, we obtain
$$
\min(s_P,s_D)^2\max(s_P,s_D)
\gtrsim \frac{n\alpha\max(s_P,s_D)}{\varepsilon^4}
\geq \frac{n^2}{\varepsilon^4}.
$$
Hence, both conditions in~\eqref{eq:baselineB} are satisfied.

\paragraph{Case 2:} In this case we have $s_P\geq s_Q>n$, and hence $\alpha=1$. We use $s_Q$ samples from each reference distribution and apply the symmetric LFHT guarantee~\eqref{eq:GP-sym}. Since the reference budget exceeds $n$, this test succeeds provided $s_D\gtrsim \frac1{\varepsilon^2}$, $s_Q\gtrsim \frac{\sqrt n}{\varepsilon^2}$, and $s_Qs_D\gtrsim \frac{n}{\varepsilon^4}$. The first condition is already part of~\eqref{eq:char}. Moreover, since case 1 was already ruled out, we have $s_D^2\lesssim \frac{n}{\varepsilon^4}$, whereas the strengthened condition on $s_D(s_Q+s_D)$ gives $s_D(s_Q+s_D)\gtrsim(\log n)^2\frac{n}{\varepsilon^4}$. Hence, the $s_D^2$ contribution can be absorbed in the hidden constants, and we obtain $s_Qs_D\gtrsim(\log n)^2\frac{n}{\varepsilon^4}$. Dividing by $s_D\lesssim \frac{\sqrt n}{\varepsilon^2}$ yields $s_Q\gtrsim(\log n)^2\frac{\sqrt n}{\varepsilon^2}$ (thus implying the same for $s_P$). Hence, all the requirements of symmetric LFHT are met.

\paragraph{Case 3:} In this case, we have $s_Q\leq n$, and also $s_P \leq n$ by the above discarding rule.  Note that the value of $\alpha$ is unchanged under such discarding: if we originally had $s_P\leq n$ then $\alpha\asymp n/s_P$, whereas if we originally had $s_P>n$, then $\alpha=1$, and after truncation $s_P=n$, so again $\alpha=n/s_P=1$. The quantities $s_Q$ and $s_D$, as well as the lower bound on $s_Qs_D$, are unaffected.

Since case 1 was ruled out, we have $s_D^2\lesssim \frac{n\alpha}{\varepsilon^4}$. On the other hand, the above-mentioned strengthened condition on $s_D(s_Q+s_D)$ gives
$$
s_D(s_Q+s_D)\gtrsim (\log n)^2\frac{n\alpha}{\varepsilon^4}.
$$
Hence, the $s_D^2$ contribution can be absorbed in the hidden constants, and we obtain $s_Qs_D\gtrsim(\log n)^2\frac{n\alpha}{\varepsilon^4}$.  Again using $s_D\lesssim \frac{\sqrt{n\alpha}}{\varepsilon^2}$, it follows that
$$
s_Q\gtrsim (\log n)^2\frac{\sqrt{n\alpha}}{\varepsilon^2}\gtrsim s_D.
$$
In particular, since $s_P\geq s_Q$, we have $s_P\geq s_Q\gtrsim s_D$.

In summary, in case 3 we may assume the following conditions:
\begin{align}\label{eq:case1}
n \geq s_P\geq s_Q\gtrsim s_D,
\qquad
\alpha\asymp\frac{n}{s_P},
\qquad
s_Qs_D\gtrsim(\log n)^2\frac{n\alpha}{\varepsilon^4}.
\end{align}
It remains to specify the test (\cref{alg:lfht-case1} below) and show that it succeeds whenever~\eqref{eq:case1} holds.

\subsection{Our Algorithm}\label{sec:case1-UB}
Our algorithm is described in \Cref{alg:lfht-case1}. As stated, it uses a random number of samples, but we later convert it to fixed sample budgets with only a constant-factor increase in sample complexity.     
It first uses independent samples from $P$ and $Q$ to \emph{flatten} all three distributions on a common expanded domain, by splitting each element into several copies and distributing its probability mass uniformly among them (see \cref{sec:flatten}).  We let $\tilde p,\tilde q,\tilde d$ denote the
flattened distributions. We will see that flattening preserves total variation distances while controlling their
largest probabilities. 

The main quantity underlying our test is
$\|\tilde p-\tilde d\|_2^2-\|\tilde q-\tilde d\|_2^2$. Its sign identifies $D$: if $D=Q$, the quantity equals $\|\tilde p-\tilde q\|_2^2>0$, whereas if $D=P$, it equals $-\|\tilde p-\tilde q\|_2^2<0$. Our algorithm constructs an unbiased estimator of this quantity. Expanding the squares, the $\|\tilde d\|_2^2$ terms cancel, leaving
$$
\|\tilde p\|_2^2-\|\tilde q\|_2^2-2\sum_j\tilde p_j\tilde d_j+2\sum_j\tilde q_j\tilde d_j.
$$
The four terms comprising $T_{\LF}$ in \Cref{alg:lfht-case1} are unbiased estimators of $\|\tilde p\|_2^2$, $\|\tilde q\|_2^2$, $\sum_j\tilde p_j\tilde d_j$, and $\sum_j\tilde q_j\tilde d_j$, respectively. Thus, $T_{\LF}$ is an unbiased estimator of the desired quantity, and its sign determines the output of the test.
\begin{algorithm2e}[h!]
    \caption{Asymmetric LFHT in the regime $n \geq s_P\geq s_Q\gtrsim s_D$}\label{alg:lfht-case1}
    \DontPrintSemicolon
    \LinesNumbered
    \KwIn{Sample access to $P,Q,D$ with $\tv(P,Q)\geq\varepsilon$ and $D\in\{P,Q\}$, budgets $s_P,s_Q,s_D$.}
    \KwOut{$D=P$ or $D=Q$.}
    Draw $\lceil s_P/2\rceil$ samples from $P$ and $\lceil s_Q/2\rceil$ from $Q$ for flattening.\\
    Construct the common split domain $[K]$ as in \cref{sec:flatten}.\\
    Independently draw testing batches of sizes $\Poi(s_P)$, $\Poi(s_Q)$, and $\Poi(s_D)$ from $P,Q,D$, and map their samples to $[K]$ as in \cref{sec:flatten}. The mapped batches are samples from $\tilde p,\tilde q,\tilde d$.\\
    Let $N^P_j,N^Q_j,N^D_j$ be the numbers of mapped testing samples taking value $j \in [K]$.\\
    $T_{\LF}\gets\sum_{j=1}^{K}\!\Big[\frac{N^P_j(N^P_j-1)}{s_P^2}-\frac{N^Q_j(N^Q_j-1)}{s_Q^2}-2\frac{N^P_jN^D_j}{s_Ps_D}+2\frac{N^Q_jN^D_j}{s_Qs_D}\Big]$\\
    \If{$T_{\LF}<0$}{\Return $D=P$}
    \Else{\Return $D=Q$}
\end{algorithm2e}

\subsubsection{Flattening}\label{sec:flatten}
We use the sample-based flattening of~\citet[Definition~2.4]{DK16}
(also used by~\citet{GP24}). Let $\mathcal S$ be the multiset of the
$\lceil s_P/2\rceil$ and $\lceil s_Q/2\rceil$ flattening samples.
For each $i\in[n]$, let $B_i$ be one plus its number of occurrences in
$\mathcal S$, and split $i$ into $B_i$ elements. Since $s_Q\leq s_P\leq n$, the expanded domain has size
\begin{align*}
    K=\sum_{i=1}^n B_i=n+\lceil s_P/2\rceil+\lceil s_Q/2\rceil\leq2n+2.
\end{align*}
Identify $[K]$ with $\{(i,a):i\in[n],\ a\in[B_i]\}$. For any
distribution $r$ on $[n]$, its flattened version is
$\tilde r(i,a)=\frac{r(i)}{B_i}.$ 
To sample from $\tilde r$, take a sample $i$ from $r$ and independently
choose $a$ uniformly from $[B_i]$. We apply this same transformation to
all three testing batches, so $\tilde d\in\{\tilde p,\tilde q\}$.
By~\citet[Fact~2.5(ii)]{DK16}, splitting preserves total variation, so $\tv(\tilde p,\tilde q)=\tv(P,Q)\geq\varepsilon$. Moreover, we have the following $\ell_{\infty}$-norm guarantees.

\begin{restatable}{lemma}{lemflat}\label{lem:flat}
    With probability at least $1-1/n$ over $\mathcal S$,
    \begin{align*}
        \|\tilde p\|_\infty\leq\frac{32\log n}{s_P},\qquad
        \|\tilde q\|_\infty\leq\frac{32\log n}{s_Q}.
    \end{align*}
\end{restatable}
The proof is given in \cref{app:ub-proofs}.

\subsubsection{Expectation, Variance, and Correctness}\label{sec:var}
Fix any realization of the flattening samples $\mathcal S$ for which
the bounds in \cref{lem:flat} hold. This also determines
$\tilde p,\tilde q,\tilde d$. The testing batches are then
independent samples from these three distributions.
Let $\Delta^2=\|\tilde p-\tilde q\|_2^2$. The vector $\tilde p-\tilde q$ has size
$K\leq2n+2\leq3n$, so
Cauchy--Schwarz gives $\|\tilde p-\tilde q\|_1\leq\sqrt{K}\,\|\tilde p-\tilde q\|_2$.
Since $\|\tilde p-\tilde q\|_1=2\tv(\tilde p,\tilde q)\geq2\varepsilon$, it follows that
\begin{align}\label{eq:delta-lb}
    \Delta^2\geq\frac{\|\tilde p-\tilde q\|_1^2}{K}
    \geq\frac{4\varepsilon^2}{3n}.
\end{align}
The following lemma states the expectation and variance needed for
Chebyshev's inequality, and is proved in \cref{app:norm-bounds}.

\begin{restatable}{lemma}{lemnorms}\label{lem:norms}
    Suppose $s_P\geq s_Q\gtrsim s_D$.
    Conditional on any $\mathcal S$ satisfying \cref{lem:flat}, the statistic
    in \cref{alg:lfht-case1} satisfies
    \begin{align*}
        \E[T_{\LF}]&=
        \begin{cases}
            -\Delta^2 & \text{if }D=P,\\
            \Delta^2 & \text{if }D=Q,
        \end{cases}\\
        \Var[T_{\LF}]&\lesssim(\log n)^2
        \left[\frac1{s_Ps_Qs_D}+\frac{\Delta^2}{s_Qs_D}\right].
    \end{align*}
\end{restatable}

By \cref{lem:norms}, $T_{\LF}$ has mean $-\Delta^2$ if $D=P$ and $\Delta^2$ if $D=Q$, so it suffices to show that its variance is small compared with $\Delta^4$.  To do so, we write the  condition on $s_D(s_Q+s_D)$ in~\eqref{eq:case1} as
\begin{align*}
    s_Qs_D\geq C(\log n)^2\frac{n\alpha}{\varepsilon^4}
    \gtrsim C(\log n)^2\frac{n^2}{s_P\varepsilon^4},
\end{align*}
where $C$ is a sufficiently large constant. Substituting into
\cref{lem:norms}, we have
\begin{align*}
    \Var[T_{\LF}]
    &\lesssim\frac{1}{C}\left[\frac{\varepsilon^4}{n^2}+\frac{s_P\varepsilon^4\Delta^2}{n^2}\right]\\
    &\leq\frac{1}{C}\left[\frac{\varepsilon^4}{n^2}+\frac{\varepsilon^2\Delta^2}{n}\right]
    &&\text{(using $s_P\leq n$ and $\varepsilon\leq1$)}\\
    &\lesssim\frac{\Delta^4}{C}
    &&\text{(using~\eqref{eq:delta-lb})}.
\end{align*}
We choose $C$ such that this is at most
$\Delta^4/6$. Under either hypothesis, an incorrect answer requires a
deviation of at least $\Delta^2$ from its mean, and hence
\begin{align*}
    \Pr[\text{incorrect sign}\mid\mathcal S]
    \leq\frac{\Var[T_{\LF}]}{\Delta^4}\leq\frac16.
\end{align*}
We proceed to summarize the other possible sources of error.  
The event of \cref{lem:flat} fails with probability at most $1/n\leq1/10$.  To implement the Poissonized testing batches using fixed sample budgets, draw $2s_P$, $2s_Q$, and $2s_D$ samples from $P$, $Q$, and $D$, respectively, and use only the first $\operatorname{Poi}(s_R)$ samples from each source $R\in\{P,Q,D\}$. If any Poisson draw exceeds the corresponding available budget, an arbitrary output is produced and this is considered an error.  Since $\Pr[\Poi(s)>2s]\leq e^{-s/3}$ and every testing budget is $\gtrsim s_D\gtrsim\frac{1}{\varepsilon^2}$, this happens with probability at most $\frac{1}{30}$. The total error is therefore at most
$1/6+1/10+1/30=3/10<1/3$, completing the upper bound.

\section{Lower Bound}\label{sec:lb}

We prove that the three conditions in \cref{thm:char} are necessary up to
constant factors for every $\varepsilon\in(0,1/8]$.
Recall that $\alpha = \max\!\left(1,\frac{n}{\max(s_P,s_D)}\right)$.
When $s_Q=0$, the problem is closeness testing between $P$ and $D$, and the required lower bound follows from~\citet{BV15}. Henceforth, we assume that $s_Q\geq1$ (and thus $s_P\geq1$ under our convention $s_P \geq s_Q$). We also assume that $s_D\geq1$, since a tester with no $D$-samples has the same output distribution under both hypotheses, and hence errs with probability at least $\frac12$ under one of them.

The first two conditions follow by a reduction from the symmetric case,
where the lower bounds of~\citet[Propositions~13--15]{GP24} hold up to constant factors. Suppose an asymmetric tester succeeds
with sample sizes $(s_P,s_Q,s_D)$, where $s_P\geq s_Q$. Given $s_P$ samples
from each of $P$ and $Q$, and $s_D$ samples from $D$, we can run this tester
using all the $P$- and $D$-samples and only the first $s_Q$ of the $Q$-samples.
This gives a symmetric tester with budgets $(s_P,s_P,s_D)$ and the same
success probability. Since $\alpha$ depends only on $s_P$ and $s_D$, it is
unchanged, so these lower bounds give
$s_D \gtrsim 1/\varepsilon^2$ and
$s_P \gtrsim \sqrt{n\alpha}/\varepsilon^2$.

We prove the third condition $s_D(s_Q + s_D) \gtrsim \frac{n\alpha}{\varepsilon^4}$ in three cases:
\begin{enumerate}
    \item \textbf{Case $s_P< s_D$.} In this case the second condition gives $s_D > s_P \gtrsim \frac{\sqrt{n\alpha}}{\varepsilon^2}$, which immediately yields $s_D(s_Q + s_D) \geq s_D^2 \gtrsim \frac{n\alpha}{\varepsilon^4}$.

    \item \textbf{Case $s_D \leq s_P$, $s_P > \frac{n}{100}$.} In this case $\alpha \leq 100$, so it suffices to show that $s_D(s_Q + s_D) \gtrsim \frac{n}{\varepsilon^4}$.

Give the tester $P$ explicitly and increase its $Q$-budget to $s_Q+s_D$.
Both changes make the problem easier. A successful asymmetric tester would
therefore solve one-distribution-known LFHT with $s_Q+s_D$ samples from $Q$
and $s_D$ from $D$. Raising the $Q$-budget to $s_Q+s_D$ makes it at least the $D$-budget, which~\citet[Remark~8]{KPS26} requires. The lower bound of
\citet[Lemma~7 and Remark~8]{KPS26}, using the construction in
\citet[Proposition~13]{GP24}, gives
$s_D(s_Q+s_D) \gtrsim \frac{n}{\varepsilon^4}$.\footnote{That bound assumes error probability $\frac{1}{10}$ rather than $\frac13$, which costs a constant factor in every budget by majority vote.}
\item 
\textbf{Case $s_D \leq s_P \leq \frac{n}{100}$.} Here $\alpha = \frac{n}{s_P}$. It suffices to show 
\begin{align}\label{inq:alfht-lb}
    s_Ps_D(s_Q+s_D)\gtrsim \frac{n^2}{\varepsilon^4},
\end{align}
which we do in the next section.
\end{enumerate}

\subsection{Proof for the Case \texorpdfstring{$s_D \leq s_P \leq \frac{n}{100}$}{sD <= sP <= n/100}}\label{sec:lb-gap}
\paragraph{Overview.}
To prove (\ref{inq:alfht-lb}), we construct two priors $\Lambda_{0}$ and $\Lambda_{1}$ over asymmetric LFHT instances $(P,Q,D)$, with $D=P$ under $\Lambda_{0}$ and $D=Q$ under $\Lambda_{1}$. With high probability, an instance drawn from either prior satisfies $\tv(P,Q)\geq\varepsilon$. We show that, when $s_{P}s_{D}(s_{Q}+s_{D})$ is sufficiently small, the sample distributions induced by $\Lambda_{0}$ and $\Lambda_{1}$ are close in total variation. On the other hand, any tester for $\RLF(\varepsilon)$ must distinguish $D=P$ from $D=Q$, and hence must distinguish the two induced sample distributions with high probability. The contradiction yields the desired lower bound.

We adapt the heavy/light construction of
\citet[Appendix~D.4]{GHP23}.
In that construction, perturbations are independent across light elements and cancel only in expectation. To simplify the calculations, we use the paired perturbations of~\citet{P08}. Within each light pair, the weight added to one element is removed from the other, so the weight vectors $w_P,w_Q,w_D$ defined below have the same total weight in every realization.
The domain is divided into independent pairs, each of which is either \emph{heavy} or \emph{light}. On a heavy pair, $P$, $Q$, and $D$ assign the same masses to the two elements, and hence samples from such a pair contain no information about whether $D=P$ or $D=Q$. Their role is instead to generate many uninformative within-pair collisions and thereby hide the informative collisions occurring on light pairs. On a light pair, $P$ is uniform, whereas $Q$ is tilted according to a hidden random sign $\eta\in\{\pm1\}$, assigning masses proportional to $(1+8\varepsilon\eta,1-8\varepsilon\eta)$.  Since $\eta$ is chosen uniformly at random, after averaging over $\eta$ the distribution of a single sample from a light pair is the same under the two hypotheses. Thus, useful information is revealed only when we get at least two samples from the same light pair. The main informative signals are a $D$--$D$ collision, which reveals whether $D$ itself is tilted, and a $D$--$Q$ collision, which reveals whether $D$ and $Q$ share the same hidden tilt. A $D$--$P$ collision is uninformative because $P$ is uniform on every light pair. At a high level, the factor $s_{D}(s_{Q}+s_{D})$ in the condition $s_{D}(s_{Q}+s_{D})\gtrsim n^{2}/(s_{P}\varepsilon^{4})$ captures the total  $D$--$D$ and $D$--$Q$ collision signal, while the additional factor $s_{P}$ reflects the masking effect created by the heavy pairs.

\paragraph{Construction.}
Let
\begin{align*}
    t_{P}:=100s_{P},\qquad
    t_{Q}:=100s_{Q},\qquad
    t_{D}:=100s_{D},
\end{align*}
and observe that $100\leq \min(t_{Q},t_{D}) \leq \max(t_{Q},t_{D})  \leq t_{P}\leq n.$

We construct two priors $\Lambda_{0}$ and $\Lambda_{1}$ on triples of probability distributions $(P,Q,D)$. The construction first generates nonnegative weight vectors $w_P,w_Q,w_D$ and then normalizes them.

 Let $N:=\lfloor n/2\rfloor$. Split $[2N]\subseteq[n]$ into the pairs $(2j-1,2j)$, $j\in[N]$. Each pair is independently declared \emph{heavy} with probability    $\gamma:=\frac{t_{P}}{4n}$
and \emph{light} otherwise. For every light pair, independently draw   $\eta_{j}\sim\operatorname{Unif}\{\pm1\}$.
We assign the weights on each pair as follows:
\begin{center}
\begin{tabular}{lcc}
    \toprule
    Pair & $(w_P(2j-1),w_P(2j))$ & $(w_Q(2j-1),w_Q(2j))$ \\
    \midrule
    Heavy & $\left(\frac{1}{t_P},\frac{1}{t_P}\right)$ & $\left(\frac{1}{t_P},\frac{1}{t_P}\right)$ \\
    \addlinespace
    Light & $\left(\frac{1}{2n},\frac{1}{2n}\right)$ & $\left(\frac{1+8\varepsilon\eta_j}{2n},\frac{1-8\varepsilon\eta_j}{2n}\right)$ \\
    \bottomrule
\end{tabular}
\end{center}
Since $\varepsilon\leq\frac{1}{8}$, these weights are nonnegative. Under $\Lambda_{0}$, set $w_D:=w_P$, whereas under $\Lambda_{1}$, set $w_D:=w_Q$. The total weight of a pair is the same under $w_P,w_Q,w_D$. Consequently, if $h$ denotes the number of heavy pairs, then all three weight vectors have the same total weight
   $ M:=\frac{2h}{t_{P}}+\frac{N-h}{n}$.
Normalizing these weight vectors gives us the distributions
\begin{align*}
    P:=\frac{w_P}{M},\qquad
    Q:=\frac{w_Q}{M},\qquad
    D:=\frac{w_D}{M},
\end{align*}
and we have $D=P$ under $\Lambda_{0}$ and $D=Q$ under $\Lambda_{1}$.

We define $E:=\big\{h\leq\frac{Nt_{P}}{2n}\big\}$, and show in the following lemma that this event has high probability and ensures valid instances.

\begin{restatable}{lemma}{lemvalidity}\label{lem:validity}
The event $E$ satisfies $\Pr[E^{c}]\leq\frac{1}{25}.$
Moreover, on $E$,
$M\geq\frac{1}{5}
    \;\text{and}\;
    \tv(P,Q)\geq\varepsilon.$
\end{restatable}

In addition, we provide the following lemma stating that a successful tester can be transferred from fixed samples to Poissonized samples.

\begin{restatable}{lemma}{lempoissonized}\label{lem:poissonized-tester}
 Suppose there exists an algorithm $\mathcal{A}$ with sample sizes $(s_{P},s_{Q},s_{D})$ satisfying $s_Q\geq1$ and $1\leq s_D \leq s_P \leq \frac{n}{100}$, and success probability at least $\frac{2}{3}$. Then there exists an algorithm which, when an instance $(P,Q,D)$ is generated according to either prior $\Lambda_{\beta}$, $\beta\in\{0,1\}$, and is given independent $\Poi(Mt_{P})$, $\Poi(Mt_{Q})$, and $\Poi(Mt_{D})$ samples from $P$, $Q$, and $D$, respectively, determines whether $D=P$ or $D=Q$ with error probability less than $\frac{2}{5}$.
\end{restatable}
The proofs of \cref{lem:validity} and \cref{lem:poissonized-tester} are deferred to \cref{app:lb-proofs}.

Let $X_i$, $Y_i$, and $Z_i$ denote the number of times element $i\in[n]$ is observed among the $\Poi(Mt_{P})$, $\Poi(Mt_{Q})$, and $\Poi(Mt_{D})$ samples drawn from $P$, $Q$, and $D$, respectively. Without loss of generality, any testing algorithm may be assumed to depend on the samples only through the counts $(X_i,Y_i,Z_i)_{i\in[n]}$, since, conditional on these counts, the ordering of the samples is independent of the underlying distributions. For each $\beta\in\{0,1\}$, let $\mathbb{P}_{\beta}$ denote the distribution of $(X_i,Y_i,Z_i)_{i\in[n]}$ when the instance is generated according to the prior $\Lambda_{\beta}$.

By the well-known Poisson thinning property, for any  generated instance and every $i\in[n]$,
\begin{align*}
X_i&\sim\Poi(Mt_{P}P(i))
=\Poi(t_{P}w_P(i)),\\
Y_i&\sim\Poi(Mt_{Q}Q(i))
=\Poi(t_{Q}w_Q(i)),\\
Z_i&\sim\Poi(Mt_{D}D(i))
=\Poi(t_{D}w_D(i)).
\end{align*}
Since the heavy/light choices and the signs are generated independently across the $N$ pairs, the sample vectors associated with distinct pairs are independent under both priors $\Lambda_0$ and $\Lambda_1$.

For each $j\in[N]$, let $V_j :=\bigl(X_{2j-1},X_{2j},Y_{2j-1},Y_{2j},Z_{2j-1},Z_{2j}\bigr)$ be the counts associated with the pair $(2j-1,2j)$.  As noted above, the vectors $V_1,\ldots,V_N$ are independent and identically distributed.  We thus suppress $j$ in the notation and use $\mathbb{P}_{\beta}^{\mathrm{pair}}$, for $\beta\in\{0,1\}$, for the distribution of $V_j$ under $\Lambda_{\beta}$.  Hence $\mathbb{P}_{\beta} = \bigl(\mathbb{P}_{\beta}^{\mathrm{pair}}\bigr)^{\otimes N}$. The divergence associated with a single pair is bounded via the following lemma, proved in \cref{sec:pairchisquare}.
\begin{restatable}{lemma}{lempairchi}\label{lem:pair-chi-square}
$
    \chi^{2}\!\left(
    \mathbb{P}_{0}^{\mathrm{pair}}
    \middle\|
    \mathbb{P}_{1}^{\mathrm{pair}}
    \right)
    \lesssim
    \frac{
    \varepsilon^{4}
    t_{P}t_{D}(t_{Q}+t_{D})
    }{n^{3}}.
$
\end{restatable}

If an asymmetric LFHT tester succeeds, \cref{lem:poissonized-tester} gives a test $\psi$ distinguishing $\mathbb{P}_0$ and $\mathbb{P}_1$ with error less than $\frac{2}{5}$ under either hypothesis. By Le Cam's testing inequality, its two error probabilities satisfy
\begin{align*}
1-\tv(\mathbb{P}_0,\mathbb{P}_1)
\leq \Pr_{\mathbb{P}_0}[\psi=1]+\Pr_{\mathbb{P}_1}[\psi=0]
<\frac{4}{5},
\end{align*}
so $\tv(\mathbb{P}_0,\mathbb{P}_1)>\frac{1}{5}$. We now show that this is impossible when the sample budget is too small.

For the $N$ independent pairs, tensorization gives the identity
\begin{align*}
1+\chi^{2}(\mathbb{P}_{0}\|\mathbb{P}_{1})
=\left(1+\chi^{2}(\mathbb{P}_{0}^{\mathrm{pair}}\|\mathbb{P}_{1}^{\mathrm{pair}})\right)^N.
\end{align*}
Thus,
\begin{align*}
\log\left(1+\chi^{2}(\mathbb{P}_{0}\|\mathbb{P}_{1})\right)
&\leq N\chi^{2}(\mathbb{P}_{0}^{\mathrm{pair}}\|\mathbb{P}_{1}^{\mathrm{pair}})
&&\bigl(\text{tensorization and } \log(1+x)\leq x\bigr)\\
&\lesssim N\frac{\varepsilon^{4}t_{P}t_{D}(t_{Q}+t_{D})}{n^{3}}
&&\bigl(\text{\cref{lem:pair-chi-square}}\bigr)\\
&\leq\frac{\varepsilon^{4}t_{P}t_{D}(t_{Q}+t_{D})}{n^{2}}
&&\bigl(N\leq n\bigr).
\end{align*}
Consequently, if $s_{P}s_{D}(s_{Q}+s_{D})\leq \frac{c n^{2}}{\varepsilon^{4}}$ for a sufficiently small constant $c$, then $\chi^{2}(\mathbb{P}_{0}\|\mathbb{P}_{1}) < \frac{4}{25}$. Hence $\tv(\mathbb{P}_{0},\mathbb{P}_{1}) \leq \frac{1}{2}\sqrt{\chi^{2}(\mathbb{P}_{0}\|\mathbb{P}_{1})} < \frac{1}{5}$, contradicting the bound above.

This proves (\ref{inq:alfht-lb}), completing the case $s_{D} \leq s_{P} \leq \frac{n}{100}$ of the third condition, and with it the lower-bound direction of \cref{thm:char}.

\section{Application: Non-Adaptive Distribution 2-Clustering}\label{sec:clustering}

We now apply the asymmetric bounds to distribution clustering. Throughout this section and \cref{app:clustering}, we assume that $n\geq10$. The formal problem formulation is as follows:

\begin{framed}
	\noindent\textbf{The Non-Adaptive Distribution 2-Clustering Problem}
	\begin{description}[leftmargin=1em]
		\item[Input:] Sample access to $k$ distributions $\{D_i\}_{i=1}^k$ over domain $[n]$, parameter $\varepsilon \in (0,1]$. The distributions are queried non-adaptively. The per-distribution sample allocation $(T_1, \ldots, T_k)$ is chosen as a (possibly random) function of $n$, $k$, $\varepsilon$, and the minority cluster size $r$ alone, and all $T_i$ samples are drawn from $D_i$ and simultaneously revealed to the learner.
		\item[Promise:] Each input distribution equals either $P$ or $Q$, where $\tv(P,Q)\geq\varepsilon$, and both occur at least once among the $k$ distributions (i.e., non-empty clusters).
		\item[Goal:] With probability at least $\frac{2}{3}$, recover the partition into the $P$- and $Q$-clusters.
		\item[Variants:] Either one cluster distribution is known, or both are unknown.
	\end{description}
\end{framed}

We write $P, Q$ for the two clusters' distributions and $\{\mathcal{I}_P, \mathcal{I}_Q\}$ for the corresponding partition. When one distribution is known, we take it to be $P$. As noted in \citet{KPS26}, it can then be taken to be the uniform distribution $U$, so the partition is $\{\mathcal{I}_U, \mathcal{I}_Q\}$.\footnote{The transformation of~\citet{Goldreich2020} maps a known distribution $P$ on $[n]$ to the uniform distribution on $[6n]$, and maps any $Q$ with $\tv(P,Q)\geq\varepsilon$ to a distribution at distance at least $\frac{\varepsilon}{3}$ from uniform. Applying it to every input changes our bounds only by constant factors.} Throughout, $r$ denotes the size of the smaller of the two clusters, so $r \in \left[1, \frac{k}{2}\right]$, and we focus on the case that $r$ is known. We note that this knowledge is only used to fix the sample allocation, but not to classify the distributions once the samples are drawn. Moreover, instead of exact knowledge of $r$, any lower bound $r_0\leq r$ suffices, with the upper bounds then holding with $r_0$ in place of $r$.  Let $T$ denote the total sample budget, so that $\sum_i T_i \leq T$ always, and $T_{\mathrm{na}}(n, k, r, \varepsilon)$ is the minimum $T$ for which some such algorithm returns the correct partition with probability $\geq \frac{2}{3}$.

Our results, summarized in \cref{tab:na-summary}, establish nearly tight upper and lower bounds on $T_{\mathrm{na}}$.  We outline the algorithm and its analysis here, and defer the details to \cref{app:clustering}. As in the adaptive case of~\citet{KPS26}, the complexity of our algorithms is the sum of the costs of two stages (here ``stages'' only refer to algorithmic steps rather than adaptive sampling steps):
\begin{itemize}[topsep=2ex,itemsep=0ex]
    \item \textbf{Stage 1.} Find one distribution from each cluster, called an exemplar.
    \item \textbf{Stage 2.} Classify every distribution via LFHT, using the two exemplars as the hypotheses.
\end{itemize}
Stage 1 uses a random subset of $\Theta\!\left(\frac{k}{r}\right)$ distributions, which contains a member of each cluster with constant probability. Because the allocation is fixed before the exemplars are identified, every distribution in this subset must receive enough samples to become an exemplar. When both distributions are unknown, one distribution of the subset, which we call the \emph{pivot}, can receive more samples than the others. Stage 2 is then an instance of asymmetric LFHT, with the pivot as $P$, the other exemplar as $Q$, and each distribution being classified as $D$.  We proceed to outline the analysis.

\begin{table}[t]
    \centering
    \renewcommand{\arraystretch}{1.4}
    \setlength{\tabcolsep}{12pt}
    \begin{tabular}{@{}lcc@{}}
        \toprule
        \textbf{Regime} & \textbf{One known distribution} & \textbf{Both distributions unknown} \\
        \midrule
        $n\gtrsim\frac{k^2}{r\varepsilon^4}$ & $\widetilde\Theta\!\left(\frac{k\sqrt n}{\sqrt r\,\varepsilon^2}\right)$ & $\widetilde\Theta\!\left(\frac{k^{2/3}n^{2/3}}{r^{1/3}\varepsilon^{4/3}}\right)$ \\
        $r\log k\lesssim n\lesssim\frac{k^2}{r\varepsilon^4}$ & $\widetilde\Theta\!\left(\frac{k\sqrt n}{\sqrt r\,\varepsilon^2}\right)$ & $\widetilde\Theta\!\left(\frac{k\sqrt n}{\sqrt r\,\varepsilon^2}\right)$ \\
        $n\lesssim r\log k$ & $\widetilde\Theta\!\left(\frac{k}{\varepsilon^2}\right)$ & $\widetilde\Theta\!\left(\frac{k}{\varepsilon^2}\right)$ \\
        \bottomrule
    \end{tabular}
    \caption{Non-adaptive sample complexity $T_{\mathrm{na}}$ for $k\geq200$ and $\varepsilon\in(0,1/8]$. The notation $\widetilde\Theta$ hides polylogarithmic factors. The bounds follow from \cref{thm:subset-c1,thm:na-lb-c1,thm:subset-c2,thm:na-lb-c2}.
    }
    \label{tab:na-summary}
\end{table}

\paragraph{One known distribution.} A random subset of $\Theta\!\left(\frac{k}{r}\right)$ distributions contains at least one member of $\mathcal{I}_Q$ with constant probability. The algorithm finds an exemplar of $\mathcal{I}_Q$ by running a uniformity test on every member of the subset, then uses the exemplar's samples as sample access to $Q$ to classify all $k$ distributions against $(U,Q)$ with asymmetric LFHT at per-test error $\frac{1}{9k}$.

For the lower bound, we provide a reduction in \cref{lem:na-reduction}, used for both variants, that turns a clustering algorithm into a tester for $D=Q$ vs.\ $D=U$. After sorting the distributions by expected allocation, an averaging argument finds a lightly sampled index $j$ such that the algorithm is unlikely to output the candidate partition formed by $j$ together with the $r-1$ least-sampled indices.  Planting $D$ at position $j$, with these other $r-1$ distributions simulating $Q$, turns the clustering algorithm into a tester whose two sample sizes are an $O\!\left(\frac{r}{k}\right)$- and an $O\!\left(\frac{1}{k}\right)$-fraction of the budget respectively, thus allowing us to invoke a known lower bound for testing.

By formalizing the above outlines, for $k\geq200$ and $\varepsilon\in(0,1/8]$, \cref{thm:subset-c1,thm:na-lb-c1} give the following sample complexity with one known distribution:
\begin{align*}
T_{\mathrm{na}}(n, k, r, \varepsilon) \;=\; \widetilde\Theta\!\left(\frac{k\sqrt n}{\sqrt r\,\varepsilon^2} + \frac{k}{\varepsilon^2}\right).
\end{align*}

\paragraph{Both distributions unknown.} The upper bound keeps the two-stage structure. Stage~1 draws a sketch from a pivot $i_A$ and runs the unequal-samples closeness test of~\citet{DK16} against every other member of a random size-$\Theta\!\left(\frac{k}{r}\right)$ subset, which yields an exemplar $i_B$ of the other cluster. Stage~2 classifies all $k$ distributions against $(D_{i_A}, D_{i_B})$ by asymmetric LFHT.

The lower bound with one known distribution also applies when both are unknown, and this is combined with a separate argument specific to the both-unknown case.  That argument again applies the reduction of \cref{lem:na-reduction}, this time into a tester for $D = P$ vs.~$D = Q$.
After ordering the indices by expected allocation, the reduction places $Q$ at the $r-1$ least-sampled indices, the planted distribution $D$ at another lightly sampled index $j$, and $P$ at all remaining indices. With $T$ samples total, the number of queries is $O(Tr/k)$ to the $Q$-oracle, $O(T/k)$ to the $D$-oracle, and (trivially) at most $T$ to the $P$-oracle.  The three-way tradeoff $s_P s_D(s_Q + s_D) \gtrsim \frac{n^2}{\varepsilon^4}$ then translates into $\frac{T^3 r}{k^2} \gtrsim \frac{n^2}{\varepsilon^4}$, thus giving an $\Omega\!\left(\frac{k^{2/3} n^{2/3}}{r^{1/3}\,\varepsilon^{4/3}}\right)$ lower bound.

We note that the three-way tradeoff is essential. In particular, planting a pair from the unequal-samples closeness-testing lower bound of~\citet{BV15} yields only $\Omega\!\left(\frac{k^{2/3} n^{2/3}}{r^{2/3}\,\varepsilon^{4/3}}\right)$, which is worse by a factor of $r^{1/3}$.

By formalizing the above outlines, under the same assumptions as the one-known setting, \cref{thm:subset-c2,thm:na-lb-c2} give the following characterization with both distributions unknown:
\begin{align*}
T_{\mathrm{na}}(n, k, r, \varepsilon) \;=\; \widetilde\Theta\!\left(\frac{k^{2/3} n^{2/3}}{r^{1/3}\,\varepsilon^{4/3}} + \frac{k\sqrt n}{\sqrt r\,\varepsilon^2} + \frac{k}{\varepsilon^2}\right).
\end{align*}
Observe that at $r = 1$ the first term is $\frac{k^{2/3} n^{2/3}}{\varepsilon^{4/3}}$, and at $r = \frac{k}{2}$ it is the LFHT-derived rate $\frac{k^{1/3} n^{2/3}}{\varepsilon^{4/3}}$ of~\citet{KPS26}.

\paragraph{Adaptivity gap.}
Let $T_{\mathrm a}$ denote the adaptive sample complexity when both cluster
distributions are unknown and $r$ is known. Comparing our bounds with
those of~\citet{KPS26} gives
$T_{\mathrm{na}}/T_{\mathrm a}=\widetilde O(k^{1/4})$
throughout the above parameter range.  As a specific example, the upper limit is attained (up to logarithmic
factors) when $n=k$, $r\asymp\sqrt{k}$, and $\varepsilon$ is constant:
\begin{align*}
    T_{\mathrm a}=\widetilde\Theta(k),\qquad
    T_{\mathrm{na}}=\widetilde\Theta(k^{5/4}).
\end{align*}
The details are given in \cref{app:adaptivity-gap}.

\section{Conclusion}\label{sec:conclusion}
In this paper, we characterized the sample budgets $(s_P,s_Q,s_D)$ for which asymmetric likelihood-free hypothesis testing is possible, up to logarithmic factors. From this characterization, we were able to derive tight sample complexity bounds for non-adaptive distribution 2-clustering.  
Among other things, we demonstrated that allowing unequal samples in $(s_P,s_Q)$ can polynomially reduce the required $s_D$; for instance, with $n$ samples from one reference distribution and $n^{3/4}$ from the other, $\widetilde O(n^{1/4})$ target samples suffice, whereas prior work  required $\Omega(\sqrt n)$.  Moreover, using asymmetric LFHT as a stepping stone, we established the optimal sample complexity of non-adaptive 2-clustering up to logarithmic factors, revealing an adaptivity gap of up to $\widetilde O(k^{1/4})$.

\section*{Acknowledgements}
This work is supported by the National Research Foundation (NRF) Singapore and the Ministry of Digital Development and Information (MDDI) under the AI Visiting Professorship (Award AIVP-2024-003).

\section*{AI Disclosure}
The lower bound in \cref{sec:lb} was developed with substantial help from an AI model (Claude). We reworked the argument ourselves to make it readable and to give intuition. The model also helped us check the proofs in most other parts of the paper. We take full responsibility for the correctness, presentation, and communication of the results.

\bibliographystyle{plainnat}
\bibliography{biblio}

\appendix
\crefalias{section}{appendix}
\crefalias{subsection}{subappendix}

\newpage
\section{Proofs for the Upper Bound}\label{app:ub-proofs}

\subsection{Flattening}

We restate and prove the flattening lemma from \cref{sec:flatten}.

\lemflat*
\begin{proof}
We first prove the bound for $P$. Let $N(i)$ count the copies of $i$ in
$\mathcal S$, so $B_i=1+N(i)$. The count $N(i)$ is a sum of independent
indicators from the two flattening batches, so
\begin{align*}
    \E[N(i)]
    =\lceil s_P/2\rceil p(i)+\lceil s_Q/2\rceil q(i)
    \geq s_Pp(i)/2.
\end{align*}
Each of the $B_i$ pieces created from $i$ has mass $p(i)/B_i$, so we bound
this ratio for every $i$.

If $s_Pp(i)<32\log n$, the original mass is already small enough:
\begin{align*}
    \frac{p(i)}{B_i}\leq p(i)<\frac{32\log n}{s_P}.
\end{align*}
Otherwise, we use concentration to ensure that $i$ is split into enough
pieces. Since $\E[N(i)]\geq s_Pp(i)/2\geq16\log n$, the Chernoff bound gives
\begin{align*}
    \Pr\!\left[N(i)<\frac{s_Pp(i)}8\right]
    \leq\exp\!\left(-\frac{9s_Pp(i)}{64}\right)
    \leq n^{-4}.
\end{align*}
Whenever $N(i)\geq s_Pp(i)/8$, the mass of each piece satisfies
\begin{align*}
    \frac{p(i)}{B_i}\leq\frac8{s_P}
    \leq\frac{32\log n}{s_P}.
\end{align*}
The same argument applies to $Q$. A union bound over the $n$ elements
for both distributions gives total failure probability at most
$2n\cdot n^{-4}\leq1/n$, proving the lemma.
\end{proof}

\subsection{Poisson Moments and Variance}

This subsection proves a general variance bound for $T_{\LF}$. We adapt
the quadratic-statistic analysis of \citet[Proposition~5 and Appendix~D]{GP24}
to independent Poisson batches with unequal sizes.

\begin{lemma}\label{lem:poisson-moments}
	Let $N \sim \Poi(\lambda)$ and, where stated, $N' \sim \Poi(\mu)$ independent of $N$. Then
	\begin{align}
		\E[N(N-1)] \;&=\; \lambda^2, \label{eq:poi-mom1}\\
		\Var[N(N-1)] \;&=\; 4\lambda^3 + 2\lambda^2, \label{eq:poi-id1}\\
		\E[NN'] \;&=\; \lambda \mu, \label{eq:poi-id3}\\
		\Var[NN'] \;&=\; \lambda^2 \mu + \lambda \mu^2 + \lambda \mu. \label{eq:poi-id2}
	\end{align}
\end{lemma}

\begin{proof}
	For every $m\geq1$, a direct calculation gives
	\begin{align}\label{eq:poi-fact}
		\E[N(N-1)\cdots(N-m+1)]
		=\sum_{k\geq m}\frac{k!}{(k-m)!}\cdot\frac{e^{-\lambda}\lambda^k}{k!}
		=\lambda^m\sum_{k\geq m}\frac{e^{-\lambda}\lambda^{k-m}}{(k-m)!}
		=\lambda^m.
	\end{align}
    We then establish the four claims as follows:
    \begin{itemize}
        \item For \eqref{eq:poi-mom1}, we simply apply~\eqref{eq:poi-fact} with $m=2$.
        \item For \eqref{eq:poi-id1}, expanding the right-hand side shows that
    	\begin{align*}
    		x^2(x-1)^2=x(x-1)(x-2)(x-3)+4x(x-1)(x-2)+2x(x-1).
    	\end{align*}
    	Applying~\eqref{eq:poi-fact} with $m=4,3,2$ to the three terms gives
    	$\E\!\left[(N(N-1))^2\right]=\lambda^4+4\lambda^3+2\lambda^2$.
    	Subtracting $(\E[N(N-1)])^2 = \lambda^4$ gives~\eqref{eq:poi-id1}.
        \item For \eqref{eq:poi-id3}, independence gives $\E[NN'] = \E[N]\E[N'] = \lambda\mu$.
        \item For \eqref{eq:poi-id2}, independence gives
	\begin{align*}
		\Var[NN'] \;=\; \E[N^2]\,\E[N'^2] - (\E[N]\E[N'])^2
		\;=\; (\lambda^2 + \lambda)(\mu^2 + \mu) - \lambda^2\mu^2,
	\end{align*}
	where we applied $\E[N^2] = \E[N(N-1)] + \E[N] = \lambda^2 + \lambda$. Expanding and simplifying gives the desired expression $\lambda^2\mu + \lambda\mu^2 + \lambda\mu$.
    \end{itemize}
\end{proof}

All expectations and variances below are conditional on the flattening
samples $\mathcal S$.

\begin{proposition}\label{prop:var}
	Let the counts
	\begin{align*}
		N^P_i \sim \Poi(s_P\, \tilde p_i), \quad N^Q_i \sim \Poi(s_Q\, \tilde q_i), \quad N^D_i \sim \Poi(s_D\, \tilde d_i)
	\end{align*}
	be independent across $i \in [K]$ and across $P, Q, D$, and let
	$T_{\LF}$ be the statistic of \cref{alg:lfht-case1}. Then
	\begin{align*}
		\E[T_{\LF}] = \|\tilde p - \tilde d\|_2^2 - \|\tilde q - \tilde d\|_2^2
	\end{align*}
	and
	\begin{multline}\label{eq:var-norm-form}
		\Var[T_{\LF}] \leq C\,\bigg[\, \frac{\|\tilde p\|_3^3 + \sum_i \tilde p_i \tilde d_i^2}{s_P} + \frac{\|\tilde q\|_3^3 + \sum_i \tilde q_i \tilde d_i^2}{s_Q} + \frac{\sum_i \tilde p_i^2 \tilde d_i + \sum_i \tilde q_i^2 \tilde d_i}{s_D} \\
		{} + \frac{\|\tilde p\|_2^2}{s_P^2} + \frac{\|\tilde q\|_2^2}{s_Q^2} + \frac{\sum_i \tilde p_i \tilde d_i}{s_P s_D} + \frac{\sum_i \tilde q_i \tilde d_i}{s_Q s_D} \,\bigg].
	\end{multline}
\end{proposition}
\begin{proof}
	Throughout the proof, we shorten $\sum_{i \in [K]}$ to $\sum_i$.

	\paragraph{Expectation.}
	For the statistic in \cref{alg:lfht-case1}, \eqref{eq:poi-mom1} gives $\frac{\E[N^P_i(N^P_i-1)]}{s_P^2}=\tilde p_i^2$ and similarly $\tilde q_i^2$ for the $Q$-term. By independence, $\frac{\E[N^P_iN^D_i]}{s_Ps_D}=\tilde p_i\tilde d_i$ and $\frac{\E[N^Q_iN^D_i]}{s_Qs_D}=\tilde q_i\tilde d_i$. Summing over $i$ gives
	\begin{align*}
		\E[T_{\LF}]
			&= \sum_i \big(\tilde p_i^2 - \tilde q_i^2 - 2\tilde p_i \tilde d_i + 2\tilde q_i \tilde d_i\big) \\
			&= \|\tilde p - \tilde d\|_2^2 - \|\tilde q - \tilde d\|_2^2.
	\end{align*}

	\paragraph{Variance.}
	Counts for distinct $i$ values are independent, so the variance of each of
	the four components below is the sum of individual variances across $i$.
	We bound $\Var[T_{\LF}]$ using the decomposition
	\begin{align}\label{eq:Tdecomp}
		T_{\LF} = T^{(P)} - T^{(Q)} - 2 T^{(PD)} + 2 T^{(QD)},
	\end{align}
	where
	\begin{align*}
		T^{(P)} &:= \sum_i \frac{N^P_i(N^P_i-1)}{s_P^2}, & T^{(Q)} &:= \sum_i \frac{N^Q_i(N^Q_i-1)}{s_Q^2}, \\
		T^{(PD)} &:= \sum_i \frac{N^P_i N^D_i}{s_P s_D}, & T^{(QD)} &:= \sum_i \frac{N^Q_i N^D_i}{s_Q s_D}.
	\end{align*}
	Applying the inequality
	$\Var[X_1+X_2+X_3+X_4]\leq4\sum_{j=1}^4\Var[X_j]$
	to the four signed terms in~\eqref{eq:Tdecomp} gives
	\begin{align}\label{eq:var-decomp}
		\Var[T_{\LF}] \lesssim \Var[T^{(P)}] + \Var[T^{(Q)}]
		+ \Var[T^{(PD)}] + \Var[T^{(QD)}].
	\end{align}

	\paragraph{Variances of $T^{(P)}$ and $T^{(Q)}$.}
	Applying~\eqref{eq:poi-id1} with $\lambda=s_P\tilde p_i$ gives
	\begin{align*}
		\Var[T^{(P)}] = \frac{1}{s_P^4}\sum_i \big(4(s_P\tilde p_i)^3 + 2(s_P\tilde p_i)^2\big) = \frac{4\,\|\tilde p\|_3^3}{s_P} + \frac{2\,\|\tilde p\|_2^2}{s_P^2}.
	\end{align*}
	An identical argument for $Q$ gives $\Var[T^{(Q)}] = \frac{4\|\tilde q\|_3^3}{s_Q} + \frac{2\|\tilde q\|_2^2}{s_Q^2}$.

	\paragraph{Variances of $T^{(PD)}$ and $T^{(QD)}$.}
	Applying~\eqref{eq:poi-id2} with $\lambda=s_P\tilde p_i$ and $\mu=s_D\tilde d_i$ gives
	\begin{align*}
		\Var[T^{(PD)}] &= \frac{1}{s_P^2 s_D^2}\sum_i \big((s_P\tilde p_i)^2(s_D\tilde d_i) + (s_P\tilde p_i)(s_D\tilde d_i)^2 + (s_P\tilde p_i)(s_D\tilde d_i)\big) \\
		&= \frac{\sum_i \tilde p_i^2 \tilde d_i}{s_D} + \frac{\sum_i \tilde p_i \tilde d_i^2}{s_P} + \frac{\sum_i \tilde p_i \tilde d_i}{s_P\, s_D},
	\end{align*}
	and analogously $\Var[T^{(QD)}] = \frac{\sum_i \tilde q_i^2 \tilde d_i}{s_D} + \frac{\sum_i \tilde q_i \tilde d_i^2}{s_Q} + \frac{\sum_i \tilde q_i \tilde d_i}{s_Q s_D}$.

	Substituting these four variance expressions into~\eqref{eq:var-decomp} and absorbing the numerical coefficients into $C$ gives~\eqref{eq:var-norm-form}.
	\end{proof}

\subsection{Norm Bounds}\label{app:norm-bounds}

We now combine \cref{prop:var} with \cref{lem:flat} to prove \cref{lem:norms}, which is restated as follows.

\lemnorms*
\begin{proof}
    By \cref{prop:var}, the expectation is $-\Delta^2$ when $D=P$ and
    $\Delta^2$ when $D=Q$. It remains to bound the variance.
    On the event of \cref{lem:flat}, we claim that
	\begin{align}\label{eq:norm-bounds}
		\|\tilde p\|_\infty \lesssim \frac{\log n}{s_P}, \qquad
		\|\tilde q\|_\infty,\ \|\tilde d\|_\infty \lesssim \frac{\log n}{s_Q}, \qquad
		\|\tilde p\|_2^2 \lesssim \frac{\log n}{s_P}, \qquad
		\|\tilde q\|_2^2,\ \|\tilde d\|_2^2 \lesssim \frac{\log n}{s_P} + \Delta^2.
	\end{align}
	The bounds on $\|\tilde p\|_\infty$ and $\|\tilde q\|_\infty$ are from \cref{lem:flat}, and the bound on $\|\tilde p\|_2^2$ follows from $\|\tilde p\|_2^2 \leq \|\tilde p\|_\infty \sum_i \tilde p_i = \|\tilde p\|_\infty$. For $\tilde q$, the triangle inequality and $(a+b)^2\leq2a^2+2b^2$ give
	\begin{align*}
		\|\tilde q\|_2^2
		=\|\tilde p+(\tilde q-\tilde p)\|_2^2
		\leq2\|\tilde p\|_2^2+2\Delta^2.
	\end{align*}
	Since $\tilde d \in \{\tilde p, \tilde q\}$ and $s_Q \leq s_P$, the bounds on $\tilde d$ follow from those on $\tilde p$ and $\tilde q$.

	We bound each term of the variance bound~\eqref{eq:var-norm-form} of \cref{prop:var} using~\eqref{eq:norm-bounds} and the following elementary inequalities:
	\begin{enumerate}[label=\textnormal{(N\arabic*)}, leftmargin=2.5em, topsep=2pt, itemsep=2pt]
		\item\label{n:third} $\|\tilde r\|_3^3 \leq \|\tilde r\|_\infty \|\tilde r\|_2^2$.
		\item\label{n:mixed} $\sum_i u_i v_i^2 \leq \|u\|_\infty \|v\|_2^2$ and $\sum_i u_i^2 v_i \leq \|v\|_\infty \|u\|_2^2$.
		\item\label{n:pd} $\sum_i \tilde p_i \tilde d_i \leq \|\tilde p\|_\infty \sum_i \tilde d_i = \|\tilde p\|_\infty$.
		\item\label{n:qd} $\sum_i \tilde q_i \tilde d_i \leq \|\tilde q\|_2 \|\tilde d\|_2 \leq \frac{1}{2}\big(\|\tilde q\|_2^2 + \|\tilde d\|_2^2\big)$.
	\end{enumerate}
	Term by term, we have the following:
	\begin{align*}
		\frac{\|\tilde p\|_3^3 + \sum_i \tilde p_i \tilde d_i^2}{s_P} &\leq \frac{\|\tilde p\|_\infty \big(\|\tilde p\|_2^2 + \|\tilde d\|_2^2\big)}{s_P} \lesssim \frac{(\log n)^2}{s_P^3} + \frac{(\log n)\,\Delta^2}{s_P^2} && \text{by \ref{n:third}, \ref{n:mixed},} \\
		\frac{\|\tilde q\|_3^3 + \sum_i \tilde q_i \tilde d_i^2}{s_Q} &\leq \frac{\|\tilde q\|_\infty \big(\|\tilde q\|_2^2 + \|\tilde d\|_2^2\big)}{s_Q} \lesssim \frac{(\log n)^2}{s_P s_Q^2} + \frac{(\log n)\,\Delta^2}{s_Q^2} && \text{by \ref{n:third}, \ref{n:mixed},} \\
		\frac{\sum_i \tilde p_i^2 \tilde d_i + \sum_i \tilde q_i^2 \tilde d_i}{s_D} &\leq \frac{\|\tilde d\|_\infty \big(\|\tilde p\|_2^2 + \|\tilde q\|_2^2\big)}{s_D} \lesssim \frac{(\log n)^2}{s_P s_Q s_D} + \frac{(\log n)\,\Delta^2}{s_Q s_D} && \text{by \ref{n:mixed},} \\
		\frac{\|\tilde p\|_2^2}{s_P^2} &\lesssim \frac{\log n}{s_P^3}, \qquad\qquad \frac{\|\tilde q\|_2^2}{s_Q^2} \lesssim \frac{\log n}{s_P s_Q^2} + \frac{\Delta^2}{s_Q^2} && \text{by \eqref{eq:norm-bounds},} \\
		\frac{\sum_i \tilde p_i \tilde d_i}{s_P s_D} &\leq \frac{\|\tilde p\|_\infty}{s_P s_D} \lesssim \frac{\log n}{s_P^2 s_D}, && \\
		\frac{\sum_i \tilde q_i \tilde d_i}{s_Q s_D} &\leq \frac{\|\tilde q\|_2^2 + \|\tilde d\|_2^2}{2\, s_Q s_D} \lesssim \frac{\log n}{s_P s_Q s_D} + \frac{\Delta^2}{s_Q s_D} && \text{by \ref{n:pd}, \ref{n:qd}.}
	\end{align*}
	Since $s_P \geq s_Q \gtrsim s_D$, the denominators satisfy $s_P^3,\ s_P s_Q^2,\ s_P^2 s_D \gtrsim s_P s_Q s_D$ and $s_P^2 \geq s_Q^2 \gtrsim s_Q s_D$, so summing the bounds gives the lemma.
\end{proof}

\section{Proofs for the Lower Bound}\label{app:lb-proofs}

We restate and prove the lemmas deferred from \cref{sec:lb-gap}, using the construction
and notation from that section.

\lemvalidity*
\begin{proof}
The number of heavy pairs satisfies
\begin{align*}
    h\sim\operatorname{Bin}\!\left(N,\frac{t_{P}}{4n}\right),
    \qquad
    \mathbb{E}[h]=\frac{Nt_{P}}{4n}.
\end{align*}
The complement $E^{c}$ is the event $h>2\mathbb{E}[h]$. By a Chernoff bound,
\begin{align*}
    \Pr[E^{c}]
    \leq
    \exp\!\left(-\frac{\mathbb{E}[h]}{3}\right).
\end{align*}
Recalling that $100\leq t_{P}\leq n$ and $N=\lfloor n/2\rfloor$, we have
\begin{align*}
    \mathbb{E}[h]
    =
    \frac{Nt_{P}}{4n}
    \geq
    \frac{99}{8}.
\end{align*}
Consequently,
\begin{align*}
    \Pr[E^{c}]
    \leq
    e^{-33/8}
    <
    \frac{1}{25}.
\end{align*}

Since $t_P\leq n$, the total mass satisfies $M\geq N/n\geq1/5$.
On $E$, we have $2h/t_P\leq N/n$ and $h\leq N/2$, which gives
\begin{align*}
    M=\frac{2h}{t_P}+\frac{N-h}{n}
    \leq\frac{N}{n}+\frac{N-h}{n}
    \leq\frac{3(N-h)}{n}.
\end{align*}
The weights coincide on heavy pairs, and each light pair contributes
$8\varepsilon/n$ to their $\ell_1$ distance. Hence,
\begin{align*}
    \tv(P,Q)
    =\frac{4(N-h)\varepsilon}{nM}
    \geq\frac{4\varepsilon}{3}
    \geq\varepsilon.
\end{align*}
\end{proof}

\lempoissonized*
\begin{proof}
Let $K_{P}$, $K_{Q}$, and $K_{D}$ denote the numbers of available samples from $P$, $Q$, and $D$, so that, conditional on the generated instance, $K_{R}\sim\Poi(Mt_{R})$ independently for $R\in\{P,Q,D\}$. Let $F$ be the event that $K_{R}\geq s_{R}$ for every $R$. The new algorithm $\mathcal{B}$ retains the first $s_{R}$ samples from each $R$ and runs $\mathcal{A}$ on them if $F$ occurs, and outputs an arbitrary answer otherwise.

The construction gives $M\geq N/n\geq\frac{1}{5}$, so $Mt_{R}\geq 20s_{R}$ for every $R$. The Poisson lower-tail bound $\Pr[K\leq(1-\delta)\lambda]\leq e^{-\delta^{2}\lambda/2}$ for $K\sim\Poi(\lambda)$, applied with $\delta=\frac{19}{20}$, gives
\begin{align*}
    \Pr[K_{R}<s_{R}]
    \leq
    \Pr\left[K_{R}\leq\frac{Mt_{R}}{20}\right]
    \leq
    e^{-9s_{R}}
    \leq
    e^{-9},
\end{align*}
so by a union bound $\Pr[F^{c}]\leq 3e^{-9}<10^{-3}$.

Conditional on a fixed generated instance and on $F$, the retained observations are exactly $s_{P}$, $s_{Q}$, and $s_{D}$ independent samples from $P$, $Q$, and $D$. On $E\cap F$, the retained samples form a valid LFHT instance, so $\mathcal{A}$ fails with probability at most $\frac{1}{3}$. Under either prior, we therefore have
\begin{align*}
    \Pr[\mathcal{B}\text{ fails}]
    \leq
    \frac{1}{3}+\Pr[E^{c}]+\Pr[F^{c}]
    \leq
    \frac{1}{3}+\frac{1}{25}+10^{-3}
    <
    \frac{2}{5}.
\end{align*}
\end{proof}

\subsection{Proof of \texorpdfstring{\cref{lem:pair-chi-square}}{Lemma \ref{lem:pair-chi-square}}}\label{sec:pairchisquare}

We restate the lemma as follows, and then prove it via three supporting lemmas.

\lempairchi*

Recall from \cref{sec:lb-gap} that $100\leq t_Q,t_D\leq t_P\leq n$.
For $R\in\{P,Q,D\}$, define
\begin{align*}
    \lambda_{R}:=\frac{t_{R}}{2n},\qquad
    \mu_{R}:=\frac{t_{R}}{t_{P}}.
\end{align*}
Hence, a heavy element of weight $1/t_P$ has expected count $\mu_R$ among
the $R$-samples. For light elements, the weights $1/(2n)$ and
$(1\pm8\varepsilon)/(2n)$ give expected counts $\lambda_R$ and
$\lambda_R(1\pm8\varepsilon)$, respectively.
In particular, $\mu_{P}=1$.

Fix a pair $(2j-1,2j)$ of domain elements from the construction.
Let $x=(x_1,x_2)$ record how often these two elements
appear among the $P$-samples, and define $y=(y_1,y_2)$ and $z=(z_1,z_2)$
similarly for the $Q$- and $D$-samples. For $x,y,z\in\mathbb{N}^{2}$, define
\begin{align*}
    H(x,y,z)
    &:={}
    \prod_{i=1}^{2}
    \poi_{1}(x_{i})
    \poi_{\mu_{Q}}(y_{i})
    \poi_{\mu_{D}}(z_{i}),\\
    L(x,y,z)
    &:={}
    \prod_{i=1}^{2}
    \poi_{\lambda_{P}}(x_{i})
    \poi_{\lambda_{Q}}(y_{i})
    \poi_{\lambda_{D}}(z_{i}).
\end{align*}
Observe that $H(x,y,z)$ equals the probability of observing these counts for a heavy
pair, and the function $L$ gives the corresponding probability when $w_P,w_Q,w_D$
all assign weight $1/(2n)$ to each element of the pair.

On a light pair, $w_Q$ instead assigns weights $(1+8\varepsilon)/(2n)$
and $(1-8\varepsilon)/(2n)$ in a random order. We account for both
possible orders with the function $g$, defined for $a,b\in\mathbb{N}$ by
\begin{align*}
    g(a,b)
    &:={}
    \frac{1}{2}(1+8\varepsilon)^{a}(1-8\varepsilon)^{b}
    +
    \frac{1}{2}(1-8\varepsilon)^{a}(1+8\varepsilon)^{b}.
\end{align*}
\begin{lemma}\label{lem:one-pair-laws}
For every $x,y,z\in\mathbb{N}^{2}$,
\begin{align*}
    \mathbb{P}_{0}^{\mathrm{pair}}(x,y,z)
    &=
    \gamma H(x,y,z)
    +(1-\gamma)L(x,y,z)g(y_{1},y_{2}),\\
    \mathbb{P}_{1}^{\mathrm{pair}}(x,y,z)
    &=
    \gamma H(x,y,z)
    +(1-\gamma)L(x,y,z)
    g(y_{1}+z_{1},y_{2}+z_{2}).
\end{align*}
\end{lemma}
\begin{proof}
We first consider a light pair. Recall that $L(x,y,z)$ is the probability
of the counts when all three weight vectors give the two elements equal
weight $1/(2n)$. Suppose that the first element has the larger $Q$-weight.
The tilt preserves the total weight of the pair and multiplies
the first and second $Q$-weights by $1+8\varepsilon$ and
$1-8\varepsilon$, respectively. The probability of observing counts
$y_1,y_2$ is therefore multiplied by
\begin{align*}
    (1+8\varepsilon)^{y_1}(1-8\varepsilon)^{y_2}.
\end{align*}
If the second element has the larger $Q$-weight, we interchange $1+8\varepsilon$ and $1-8\varepsilon$.

Recall the priors $\Lambda_0$ and $\Lambda_1$ defined in \cref{sec:lb-gap}.    
Under $\Lambda_0$, both $w_P$ and $w_D$ give each element weight $1/(2n)$.
Averaging over the two possible orders of the $Q$-weights gives the count
probability $L(x,y,z)g(y_1,y_2)$. 
Under $\Lambda_1$, we have $D=Q$, so the $D$-samples use the same orientation
as the $Q$-samples. The first element is observed $y_1+z_1$ times across
these two batches, and the second is observed $y_2+z_2$ times. The same
argument therefore gives $L(x,y,z)g(y_1+z_1,y_2+z_2)$.

A heavy pair has count distribution $H$ under both priors. Combining the
heavy and light cases, which occur with probabilities $\gamma$ and
$1-\gamma$, gives the stated formulas.
\end{proof}
For $y=(y_{1},y_{2})$ and $z=(z_{1},z_{2})$, let
\begin{align*}
    |y|:=y_{1}+y_{2},\qquad
    |z|:=z_{1}+z_{2}.
\end{align*}
To compare the two light-pair probabilities, let $B$ be the difference
between the two values of $g$:
\begin{align*}
    B(y,z)
    &:={}
    g(y_{1},y_{2})
    -g(y_{1}+z_{1},y_{2}+z_{2}).
\end{align*}
\begin{lemma}\label{lem:perturbation}
For every $y,z\in\mathbb{N}^{2}$,
\begin{enumerate}[label=\textnormal{(\arabic*)},leftmargin=2em]
    \item $B(y,z)=0$ unless $|z|\geq1$ and $|y|+|z|\geq2$.
    \item $|B(y,z)| \lesssim \varepsilon^{2}2^{|y|+|z|}.$
\end{enumerate}
\end{lemma}
\begin{proof}
If $|z|=0$, the two terms defining $B(y,z)$ coincide. If
$|y|+|z|\leq1$, both values of $g$ equal one, since $g(a,b)=1$ for
$a+b\leq1$. This proves the first claim.

For the second claim, view $B(y,z)$ as a polynomial in $v=8\varepsilon$.
Its constant terms cancel, and averaging over the two signs removes all
odd powers, so every surviving term has degree at least two.
By the binomial theorem, the sum of the absolute coefficients of
$(1+v)^a(1-v)^b$ is at most
\begin{align*}
    \sum_{r=0}^a\sum_{s=0}^b\binom ar\binom bs=2^{a+b}.
\end{align*}
Averaging over the signs cannot increase this sum. Hence, the sum of the
absolute coefficients of $B(y,z)$ is at most
$2^{|y|}+2^{|y|+|z|}$. Since $v\leq1$ and all surviving powers are at
least two,
\begin{align*}
    |B(y,z)|
    \leq(8\varepsilon)^2\bigl(2^{|y|}+2^{|y|+|z|}\bigr)
    \leq2(8\varepsilon)^2 2^{|y|+|z|} \lesssim \varepsilon^{2}2^{|y|+|z|}.
\end{align*}
\end{proof}

\begin{lemma}\label{lem:two-coordinate-weights}
For every $R\in\{P,Q,D\}$ and $m\in\{0,1,2\}$,
\begin{align*}
    \sum_{|u|\geq m}4^{|u|}\prod_{i=1}^{2}
    \frac{\poi_{\lambda_R}(u_i)^2}{\poi_{\mu_R}(u_i)}
    \lesssim\left(\frac{t_Rt_P}{n^2}\right)^m.
\end{align*}
\end{lemma}
\begin{proof}
A direct substitution of the Poisson probability formula gives
\begin{align*}
    \frac{\poi_{\lambda_R}(k)^2}{\poi_{\mu_R}(k)}
    =e^{\mu_R-2\lambda_R}\frac{(t_Rt_P/(4n^2))^k}{k!}.
\end{align*}
It follows that
\begin{align*}
    \sum_{|u|\geq m}4^{|u|}\prod_{i=1}^{2}
    \frac{\poi_{\lambda_R}(u_i)^2}{\poi_{\mu_R}(u_i)}
    &=e^{2\mu_R-4\lambda_R}\sum_{k\geq m}
      \left(\frac{t_Rt_P}{n^2}\right)^k
      \sum_{a+b=k}\frac{1}{a!b!}\\
    &=e^{2\mu_R-4\lambda_R}\sum_{k\geq m}\frac{(2t_Rt_P/n^2)^k}{k!}
    &&\left(\sum_{a+b=k}\frac{1}{a!b!}=\frac{2^k}{k!}\right)\\
    &\leq e^{2\mu_R-4\lambda_R+2t_Rt_P/n^2}\frac{(2t_Rt_P/n^2)^m}{m!}
    &&\bigl(k!\geq m!(k-m)!\bigr)\\
    &\lesssim\left(\frac{t_Rt_P}{n^2}\right)^m
    &&\bigl(t_R\leq t_P\leq n\bigr).
\end{align*}
\end{proof}

\begin{proof}[Proof of \cref{lem:pair-chi-square}]
By \cref{lem:one-pair-laws}, we have
\begin{align*}
    \mathbb P_0^{\mathrm{pair}}(x,y,z)-\mathbb P_1^{\mathrm{pair}}(x,y,z)
    &=(1-\gamma)L(x,y,z)B(y,z),\\
    \mathbb P_1^{\mathrm{pair}}(x,y,z)&\geq\gamma H(x,y,z).
\end{align*}
Substituting this into the definition
of $\chi^2$ gives
\begin{align*}
    \chi^2(\mathbb P_0^{\mathrm{pair}}\|\mathbb P_1^{\mathrm{pair}})
    &=\sum_{x,y,z}
       \frac{\bigl(\mathbb P_0^{\mathrm{pair}}(x,y,z)-\mathbb P_1^{\mathrm{pair}}(x,y,z)\bigr)^2}
            {\mathbb P_1^{\mathrm{pair}}(x,y,z)}\\
    &\leq\frac{(1-\gamma)^2}{\gamma}\sum_{x,y,z}
       \frac{L(x,y,z)^2}{H(x,y,z)}B(y,z)^2\\
       &\leq\frac{1}{\gamma}\sum_{x,y,z}
       \frac{L(x,y,z)^2}{H(x,y,z)}B(y,z)^2.
\end{align*}
By the definitions of $L$ and $H$,
\begin{align*}
    \frac{L(x,y,z)^2}{H(x,y,z)}
    =\prod_{i=1}^2
    \frac{\poi_{\lambda_P}(x_i)^2}{\poi_1(x_i)}
    \frac{\poi_{\lambda_Q}(y_i)^2}{\poi_{\mu_Q}(y_i)}
    \frac{\poi_{\lambda_D}(z_i)^2}{\poi_{\mu_D}(z_i)}.
\end{align*}
Since $B(y,z)$ does not depend on $x$, we can sum over the $P$-counts
first. Applying \cref{lem:two-coordinate-weights} with $R=P$ (and hence $\mu_R = 1$) and $m=0$ gives
\begin{align*}
    \sum_x\prod_{i=1}^2\frac{\poi_{\lambda_P}(x_i)^2}{\poi_1(x_i)}
    \leq\sum_x4^{|x|}\prod_{i=1}^2
    \frac{\poi_{\lambda_P}(x_i)^2}{\poi_1(x_i)}\lesssim1.
\end{align*}
Since $\gamma=\frac{t_P}{4n}$, we therefore obtain
\begin{align*}
    \chi^2(\mathbb P_0^{\mathrm{pair}}\|\mathbb P_1^{\mathrm{pair}})
    &\lesssim\frac n{t_P}\sum_{y,z}
       \left(\prod_{i=1}^2
       \frac{\poi_{\lambda_Q}(y_i)^2}{\poi_{\mu_Q}(y_i)}
       \frac{\poi_{\lambda_D}(z_i)^2}{\poi_{\mu_D}(z_i)}\right)B(y,z)^2.
\end{align*}

By \cref{lem:perturbation}, $B(y,z)=0$ unless $|z|\geq1$ and
$|y|+|z|\geq2$. Hence, if $|z|\geq2$, there is no restriction on $y$.
If $|z|=1$, we need $|y|\geq1$.
The same lemma gives $B(y,z)^2\lesssim\varepsilon^4 4^{|y|}4^{|z|}$.
After this substitution, the summand is a product of a function of $y$
and a function of $z$, so the two sums can be evaluated separately.
For $|z|\geq2$, \cref{lem:two-coordinate-weights} gives
\begin{align*}
    &\sum_{\substack{y,z\\|z|\geq2}}
       \left(\prod_{i=1}^2
       \frac{\poi_{\lambda_Q}(y_i)^2}{\poi_{\mu_Q}(y_i)}
       \frac{\poi_{\lambda_D}(z_i)^2}{\poi_{\mu_D}(z_i)}\right)B(y,z)^2\\
    &\qquad\lesssim\varepsilon^4
       \underbrace{\left(\sum_y4^{|y|}\prod_{i=1}^2
       \frac{\poi_{\lambda_Q}(y_i)^2}{\poi_{\mu_Q}(y_i)}\right)}_{\lesssim1\quad(m=0)}
       \underbrace{\left(\sum_{|z|\geq2}4^{|z|}\prod_{i=1}^2
       \frac{\poi_{\lambda_D}(z_i)^2}{\poi_{\mu_D}(z_i)}\right)}_{\lesssim(t_Dt_P/n^2)^2\quad(m=2)}\\
    &\qquad\lesssim\varepsilon^4\left(\frac{t_Dt_P}{n^2}\right)^2.
\end{align*}
For $|z|=1$ and $|y|\geq1$, every term is nonnegative, so enlarging
the $D$-sum from $|z|=1$ to $|z|\geq1$ can only increase it.
We can then apply \cref{lem:two-coordinate-weights} with $m=1$ to both sums:
\begin{align*}
    &\sum_{\substack{|y|\geq1\\|z|=1}}
       \left(\prod_{i=1}^2
       \frac{\poi_{\lambda_Q}(y_i)^2}{\poi_{\mu_Q}(y_i)}
       \frac{\poi_{\lambda_D}(z_i)^2}{\poi_{\mu_D}(z_i)}\right)B(y,z)^2\\
    &\qquad\lesssim\varepsilon^4
       \underbrace{\left(\sum_{|y|\geq1}4^{|y|}\prod_{i=1}^2
       \frac{\poi_{\lambda_Q}(y_i)^2}{\poi_{\mu_Q}(y_i)}\right)}_{\lesssim t_Qt_P/n^2\quad(m=1)}
       \underbrace{\left(\sum_{|z|\geq1}4^{|z|}\prod_{i=1}^2
       \frac{\poi_{\lambda_D}(z_i)^2}{\poi_{\mu_D}(z_i)}\right)}_{\lesssim t_Dt_P/n^2\quad(m=1)}\\
    &\qquad\lesssim\varepsilon^4\frac{t_Qt_P}{n^2}\frac{t_Dt_P}{n^2}.
\end{align*}
Adding these two bounds gives
\begin{align*}
    \chi^2(\mathbb P_0^{\mathrm{pair}}\|\mathbb P_1^{\mathrm{pair}})
    \lesssim\frac n{t_P}\varepsilon^4
    \left[\left(\frac{t_Dt_P}{n^2}\right)^2
    +\frac{t_Qt_P}{n^2}\frac{t_Dt_P}{n^2}\right] =\frac{\varepsilon^4t_Pt_D(t_Q+t_D)}{n^3}.
\end{align*}
\end{proof}

\subsection{Average-Case Lower Bound}

For later use, we state the following remark.

\begin{remark}[Average-case form of \eqref{inq:alfht-lb}]\label{rem:gap-avg}
We claim that the lower bound~\eqref{inq:alfht-lb} also holds if the tester is required to have error probability at most $\frac{1}{3}$ on average under each conditional prior $\Lambda_\beta(\cdot\mid E)$, where $\beta \in \{0,1\}$, $E$ is the high-probability event in \cref{lem:validity}, and $\Lambda_\beta$ are the priors specified in Section \ref{sec:lb-gap}.  The analysis is unchanged except for the Poissonization step in \cref{lem:poissonized-tester}.  For that step, let $e(I)$ be the error probability of the tester on a fixed instance $I$, and let $F$ be the sufficient-samples event from the proof of \cref{lem:poissonized-tester}. Conditional on $I$, the sample counts are independent of the retained sample values, and hence the contribution to the conditional error probability on the event $F$ satisfies
\begin{align*}
    \mathbb{E}_{I\sim\Lambda_\beta(\cdot\mid E)}
    \bigl[e(I)\Pr(F\mid I)\bigr]
    \leq
    \mathbb{E}_{I\sim\Lambda_\beta(\cdot\mid E)}[e(I)]
    \leq\frac13.
\end{align*}
Adding $\Pr(E^c)\leq1/25$ and $\Pr(F^c\mid E)<10^{-3}$  then gives the same error bound as before, namely $1/3+1/25+10^{-3}<2/5$, under each unconditioned prior. The rest of the proof in \cref{sec:lb-gap} is unchanged.
\end{remark}

\section{Non-Adaptive Distribution Clustering}\label{app:clustering}

This appendix gives the algorithms and proofs for the 2-clustering results
in \cref{sec:clustering}.

\subsection{Upper Bounds with One Known Distribution}\label{sec:case1}

This subsection proves the upper bounds with one known distribution summarized in \cref{sec:clustering}, using the model and quantities $r$, $T$, $T_{\mathrm{na}}$ defined therein. The matching lower bound is proved in \cref{app:clustering-lower-known}.  Recall from \cref{sec:clustering} and \cite{KPS26} that it suffices to consider the case that the known distribution is uniform, and accordingly denoted by $U$.

\subsubsection{Baseline}

Uniformity testing in the sparse regime was studied by \citet{P08}. We use the identity tester of \citet[Proposition~2.7]{DK16}, specialized to uniformity, which applies throughout our parameter range.

\begin{restatable}{theorem}{thmnaiveone}\label{thm:naive_1}
$T_{\mathrm{na}}(n, k, r, \varepsilon) = O\!\left(\frac{k\sqrt{n}\log k}{\varepsilon^2}\right)$ when one distribution is known.
\end{restatable}
\begin{proof}
	Since the known distribution is $U$, each input distribution satisfies $D_i = U$ or $\tv(D_i, U) \geq \varepsilon$. Run this uniformity test on each input distribution at per-test error $\frac{1}{9k}$, costing $O\!\left(\frac{\sqrt n \log k}{\varepsilon^2}\right)$ samples per input distribution, and assign $i$ to $\mathcal{I}_U$ or $\mathcal{I}_Q$ by the outcome. A union bound over the $k$ input distributions gives success $\geq \frac{8}{9}$, for a total of $O\!\left(\frac{k\sqrt n \log k}{\varepsilon^2}\right)$.
\end{proof}

\subsubsection{Improved Upper Bound}

With the known distribution being $U$, the task is to find the hidden partition
$\{\mathcal{I}_U, \mathcal{I}_Q\}$. \Cref{alg:case1} implements the two stages sketched in \cref{sec:clustering}.

\begin{algorithm2e}[h!]
\caption{Non-adaptive 2-clustering with one known distribution.}
\label{alg:case1}
\LinesNumbered
\KwIn{Parameters $s_0, s_1, s_2 \in \mathbb{N}$.}
\KwOut{Partition of $[k]$.}
Set $\ell=\lceil18\log(9k)\rceil$\;
Pick a uniformly random subset $S \subset [k]$ with $|S| = \min\left\{k,\ \left\lceil \frac{3k}{r} \right\rceil\right\}$\;
Draw $s_0 + \ell(s_1 + s_2)$ samples from each $i \in S$ and $\ell s_2$ samples from each
$i \in [k] \setminus S$\;
\For{$i \in S$}{
  Run a uniformity test on the first $s_0$ samples from $i$, with per-test error
  $\frac{1}{9|S|}$\;
}
Let $i^*$ be any rejected distribution if one exists. Otherwise, return an arbitrary
size-$(r, k-r)$ partition (or alternatively declare failure)\;
\For{$i \in [k]$}{
  Run asymmetric LFHT for $(U,D_{i^*})$ in $\ell$ independent repetitions, each using samples generated from $U$ and disjoint batches of $s_1$ samples from $i^*$ and $s_2$ from $D_i$, and classify $D_i$ by majority vote\;
}
\Return the partition
\end{algorithm2e}

\begin{restatable}{theorem}{thmsubsetcone}\label{thm:subset-c1}
For $r \in \left[1, \frac{k}{2}\right]$, Algorithm~\ref{alg:case1} solves 2-clustering
with success probability $\geq 2/3$ using
\begin{align*}
T = \begin{cases}
\widetilde O\!\left(\dfrac{k\sqrt{n}}{\sqrt{r}\,\varepsilon^2}\right) & \text{if } n \gtrsim r\log k, \\
\widetilde O\!\left(\dfrac{k\log k}{\varepsilon^2}\right) & \text{if } n \lesssim r\log k
\end{cases}
\end{align*}
samples.
\end{restatable}

\begin{proof}
Stage~1 uses $s_0=O(\sqrt n\log k/\varepsilon^2)$ samples per member of $S$,
for a total of $\widetilde O(k\sqrt n/(r\varepsilon^2))$. With probability at
least $7/9$, the subset contains a member of $\mathcal{I}_Q$ and all its
uniformity tests are correct, so $i^*\in\mathcal{I}_Q$.

For Stage~2, use $U$ as the more heavily sampled reference in
\cref{thm:char}. Since $U$ is known, its samples can be generated directly
without using the input-distribution budget. Generating enough of them
makes $\alpha=1$ and satisfies the condition on the larger reference budget.
By \cref{thm:char}, it suffices to take any pair $(s_1,s_2)$ satisfying the following and inflate it by a polylogarithmic factor:
\begin{align*}
    s_2\gtrsim\frac1{\varepsilon^2},\qquad
    s_2(s_1+s_2)\gtrsim\frac n{\varepsilon^4}.
\end{align*}
Accordingly, we choose
\begin{align*}
    s_2=\widetilde\Theta\!\left(
        \max\left\{\frac1{\varepsilon^2},
                    \frac{\sqrt n}{\sqrt r\,\varepsilon^2}\right\}\right),
    \qquad s_1=r s_2.
\end{align*}
The two terms in the maximum ensure the conditions $s_2\gtrsim1/\varepsilon^2$ and
$s_1s_2=r s_2^2\gtrsim n/\varepsilon^4$ prior to the polylogarithmic increase, and the $\widetilde\Theta$ notation captures the required increase.

The $\ell s_1$ samples must be reserved for every member of $S$, since $i^*$
is selected only after sampling. Since $\ell=\Theta(\log k)$, Stage~2 costs
\begin{align*}
    \ell\bigl(|S|s_1+ks_2\bigr)\asymp\ell\left(\frac{k}{r}s_1+ks_2\right)
    \asymp\ell ks_2
    =\widetilde O\!\left(\frac{k\sqrt n}{\sqrt r\,\varepsilon^2}
                      +\frac{k}{\varepsilon^2}\right).
\end{align*}
This also covers Stage~1 since $r\geq1$, and gives the two stated regimes.
The Stage~2 batches are independent of the samples used to choose $i^*$.
For each classification, the $\ell$ repetitions use disjoint samples and
independent internal randomness. Hoeffding's inequality bounds the
majority-vote error by $e^{-\ell/18}\leq1/(9k)$. The reference batches may
be reused across classifications. A union bound over the $k$
classifications gives Stage~2 failure at most $1/9$, and total failure
at most $2/9+1/9=1/3$.
\end{proof}

\subsection{Upper Bounds with Both Distributions Unknown}\label{sec:case2}

The non-adaptive setting is as in \cref{sec:clustering}. Here both cluster distributions $P, Q$ are
unknown, and the task is to find the hidden partition $\{\mathcal{I}_P, \mathcal{I}_Q\}$. We prove the upper bounds here and the matching lower bound in \cref{app:clustering-lower-unknown}.

\subsubsection{Baseline}

Here the analog of \cref{thm:naive_1} is given as follows.

\begin{restatable}{theorem}{thmnaivetwo}\label{thm:naive_2}
$T_{\mathrm{na}}(n, k, r, \varepsilon) = O\!\left(k\log k \cdot \max\!\left(\tfrac{n^{2/3}}{\varepsilon^{4/3}},\,\tfrac{\sqrt{n}}{\varepsilon^2}\right)\right)$
when both distributions are unknown.
\end{restatable}

\begin{proof}
    Fix an arbitrary index $i_A \in [k]$, and let its distribution $D_{i_A}$ be the one given more samples, recalling that we refer to $i_A$ as the pivot.  By assumption in the problem formulation, the other $k-1$ distributions all satisfy $D_i = D_{i_A}$ or $\tv(D_i, D_{i_A}) \geq \varepsilon$. Run a closeness test between $D_{i_A}$ and each $D_i$ \citep{CDVV14} at per-test error $\frac{1}{9k}$, costing $O\!\left(\max\!\left(\frac{n^{2/3}}{\varepsilon^{4/3}},\, \frac{\sqrt n}{\varepsilon^2}\right)\log k\right)$ samples per test (the pivot's samples are drawn once and reused), and cluster $[k]$ according to the outcomes. A union bound over the $k - 1$ tests gives success $\geq \frac{8}{9}$, for a total of $O\!\left(k\log k \max\!\left(\frac{n^{2/3}}{\varepsilon^{4/3}},\, \frac{\sqrt n}{\varepsilon^2}\right)\right)$.
\end{proof}

\subsubsection{Improved Upper Bounds}\label{sec:case2-improved}

We improve on the baseline bound with \cref{alg:case2}, which implements the two stages sketched in \cref{sec:clustering}. A random size-$\Theta\!\left(\frac{k}{r}\right)$ subset $S$ supplies an exemplar of each cluster, an element $i_A$ of $S$ serving as the pivot, and asymmetric LFHT then classifies all $k$ distributions against the two exemplars.

Stage~1 uses the closeness tester of~\citet{DK16} with unequal sample sizes, which we denote by $\UST$. Given $m$ samples from one distribution and $m'\leq m$ from another, it decides whether the two are equal or at total variation distance at least $\varepsilon$, with error at most $\frac13$, provided that
\begin{align*}
m'\gtrsim\frac{\sqrt n}{\varepsilon^2}\qquad\text{and}\qquad m'^2m\gtrsim\frac{n^2}{\varepsilon^4}.
\end{align*}
These are the conditions in~\eqref{eq:baselineB}. Taking the majority of $O(\log k)$ independent runs reduces the error to $\frac{1}{9|S|}$, and we will use this guarantee in the proof below.

\begin{algorithm2e}[h!]
\caption{Non-adaptive 2-clustering with both distributions unknown.}
\label{alg:case2}
\LinesNumbered
\KwIn{Parameters $s_a, s_b, s_P, s_Q, s_D \in \mathbb{N}$.}
\KwOut{Partition of $[k]$.}
Set $\ell=\lceil18\log(9k)\rceil$\;
Pick a uniformly random subset $S \subset [k]$ with $|S| = \min\left\{k,\ \left\lceil \frac{3k}{r} \right\rceil\right\}$, and let $i_A$ be any element of $S$\;
Draw $s_a + \ell(s_P + s_D)$ samples from $i_A$, $s_b + \ell(s_Q + s_D)$ samples from each
$i \in S \setminus \{i_A\}$, and $\ell s_D$ samples from each $i \in [k] \setminus S$\;
\For{$i \in S \setminus \{i_A\}$}{
  Run $\UST$ for $(D_{i_A}, D_i)$ using the $s_a$ samples from $i_A$ as sketch and the
  $s_b$ samples from $i$ as test, at per-test error $\frac{1}{9|S|}$\;
}
Let $i_B$ be any $i \in S \setminus \{i_A\}$ whose test rejected if one exists.
Otherwise return an arbitrary size-$(r, k-r)$ partition (or alternatively declare failure)\;
\For{$i \in [k]$}{
  Run asymmetric LFHT for $(D_{i_A},D_{i_B})$ in $\ell$ independent repetitions, each using disjoint batches of $s_P$ samples from $i_A$, $s_Q$ from $i_B$, and $s_D$ from $D_i$, and classify $D_i$ by majority vote\;
}
\Return the partition
\end{algorithm2e}

\begin{restatable}{theorem}{thmsubsetctwo}\label{thm:subset-c2}
For $r \in \left[1, \frac{k}{2}\right]$, Algorithm~\ref{alg:case2} solves 2-clustering
with success probability $\geq 2/3$ using
\begin{align*}
T = \begin{cases}
\widetilde O\!\left(\dfrac{k^{2/3}\, n^{2/3}}{r^{1/3}\,\varepsilon^{4/3}}\right) & \text{if } n \gtrsim \dfrac{k^2}{r\varepsilon^4}, \\
\widetilde O\!\left(\dfrac{k\sqrt{n}}{\sqrt{r}\,\varepsilon^2}\right) & \text{if } r\log k \lesssim n \lesssim \dfrac{k^2}{r\varepsilon^4}, \\
\widetilde O\!\left(\dfrac{k\log k}{\varepsilon^2}\right) & \text{if } n \lesssim r\log k
\end{cases}
\end{align*}
samples.
\end{restatable}

\begin{proof}
In Stage~1, choose
\begin{align*}
s_b=\widetilde\Theta\!\left(\max\left\{\frac{\sqrt n}{\varepsilon^2},\ \left(\frac{n^2r}{k\varepsilon^4}\right)^{1/3}\right\}\right),
\qquad s_a=\frac{k}{r}\,s_b,
\end{align*}
where the logarithmic factor covers the repetitions needed for error $\frac{1}{9|S|}$. Since $s_a\geq s_b$, the guarantee of $\UST$ requires $s_b\gtrsim\frac{\sqrt n}{\varepsilon^2}$ and $s_b^2s_a=\frac{k}{r}s_b^3\gtrsim\frac{n^2}{\varepsilon^4}$, which both hold. The Stage~1 cost is
\begin{align*}
s_a + |S|\, s_b \asymp \frac{k}{r}\, s_b = \widetilde O\!\left(\max\!\left(\left(\frac{nk}{r\varepsilon^2}\right)^{2/3},\; \frac{k\sqrt n}{r\varepsilon^2}\right)\right).
\end{align*}
Unless $S = [k]$, we have $|S| \geq \frac{3k}{r}$, and each cluster has at least $r$ members, so a given cluster has no member in $S$ with probability at most $\left(1 - \frac{r}{k}\right)^{|S|} \leq e^{-3} < \frac{1}{18}$. Hence $S$ contains a member of each cluster except with probability $\frac{1}{9}$, and a union bound over the $|S|-1$ tests, each at error $\frac{1}{9|S|}$, gives total Stage~1 failure $\leq \frac{2}{9}$. On Stage~1 success, some $i \in S \setminus \{i_A\}$ lies in the cluster opposite $i_A$ and its test rejects, while every same-cluster test accepts, so $i_B$ is defined and lies in the opposite cluster.

In Stage~2, apply \cref{thm:char} with the pivot $i_A$ as the more heavily
sampled reference, so one run uses sample sizes $(s_P,s_Q,s_D)$.
Choose
\begin{align*}
    s_D&=\widetilde\Theta\!\left(\max\left\{
        \frac1{\varepsilon^2},\quad
        \frac{\sqrt n}{\sqrt r\,\varepsilon^2},\quad
        \left(\frac{n^2}{rk\varepsilon^4}\right)^{1/3}
        \right\}\right),\\
    s_Q&=r s_D,\qquad s_P=k s_D,
\end{align*}
where the polylogarithmic factor in $\widetilde\Theta$ is the one required by \cref{thm:char}.
The second and third terms in the maximum give the following conditions (before the required logarithmic increase):
\begin{align*}
    s_Qs_D=r s_D^2\gtrsim\frac n{\varepsilon^4},\qquad
    s_Ps_Qs_D=kr s_D^3\gtrsim\frac{n^2}{\varepsilon^4},
\end{align*}
which together imply $s_Qs_D\gtrsim n\alpha/\varepsilon^4$ for
$\alpha=\max\{1,n/s_P\}$. Moreover, $s_P\geq s_Q\geq s_D$ gives
$s_P^2\geq s_Qs_D\gtrsim n\alpha/\varepsilon^4$, and the first term
ensures $s_D\gtrsim1/\varepsilon^2$. The LFHT test analyzed in
\cref{sec:ub} therefore has error probability at most $1/3$ with these budgets.

Fix the choices and samples from Stage~1, and suppose it succeeds.
Then $i_A$ and $i_B$ belong to different clusters. The reserved Stage~2
samples are independent of the Stage~1 samples. For each input $D_i$,
the $\ell$ LFHT runs use disjoint batches and independent internal
randomness. Hoeffding's inequality therefore bounds the majority-vote
error by $e^{-\ell/18}\leq1/(9k)$. We use the same reference samples
for all $k$ inputs. By a union bound, the probability of any wrong
classification is at most $k\cdot1/(9k)=1/9$.

We must allocate samples before choosing $i_B$, so we reserve
$\ell s_Q$ reference samples for
every candidate in $S\setminus\{i_A\}$. Since $\ell=\Theta(\log k)$, the
Stage~2 cost is
\begin{align*}
    \ell\bigl(s_P+(|S|-1)s_Q+ks_D\bigr)
    &\asymp\ell\left(s_P+\frac{k}{r}s_Q+ks_D\right)
    \asymp\ell k s_D\\*
    &=\widetilde O\!\left(
        \frac{k^{2/3}n^{2/3}}{r^{1/3}\varepsilon^{4/3}}
        +\frac{k\sqrt n}{\sqrt r\,\varepsilon^2}
        +\frac{k}{\varepsilon^2}\right).
\end{align*}
Since $r\geq1$, the Stage~1 terms satisfy
\begin{align*}
    \left(\frac{nk}{r\varepsilon^2}\right)^{2/3}
    &\leq\frac{k^{2/3}n^{2/3}}{r^{1/3}\varepsilon^{4/3}},\qquad
    \frac{k\sqrt n}{r\varepsilon^2}
    \leq\frac{k\sqrt n}{\sqrt r\,\varepsilon^2}.
\end{align*}
Thus Stage~1 is covered by the same bound. Comparing the first two terms
gives the threshold $n=k^2/(r\varepsilon^4)$. For $n\lesssim r\log k$,
all three terms are $\widetilde O(k/\varepsilon^2)$. This gives the stated regimes. The total
failure probability is at most $2/9+1/9=1/3$.
\end{proof}

\subsection{Lower Bounds}\label{app:clustering-lower}

The two lower bounds use a common reduction, \cref{lem:na-reduction}, from a clustering algorithm to an LFHT tester. A clustering algorithm with total budget $T$ gives an LFHT tester with $s_P=T$, $s_Q=O(\frac{Tr}{k})$, and $s_D=O(\frac{T}{k})$. We then apply the LFHT lower bounds with one known distribution and with both unknown.

We may assume that the clustering algorithm succeeds with probability at least $\frac{99}{100}$.  For instance, repeating it a constant number of times and taking the majority output achieves this, multiplies $T$ by a constant, and keeps the algorithm non-adaptive.

\subsubsection{The Reduction from LFHT}

To turn a clustering algorithm into an LFHT tester, we place the unknown distribution $D$ at a lightly sampled index and use the output partition to distinguish $D=P$ from $D=Q$. This follows the ideas of~\citet[Appendix~E.2, Lemmas~22--23]{KPS26}.

We order the indices by expected number of samples. The tester places copies of $Q$ at the first $r-1$ indices, the unknown distribution $D$ at one more index $j$ with few expected samples, and copies of $P$ everywhere else. If $D=Q$, the small cluster is the first $r-1$ indices together with $j$, and the algorithm returns this partition with high probability. If $D=P$, the instance has only $r-1$ copies of $Q$ and need not satisfy the clustering promise. We do not assume correctness on this instance. Instead, an averaging argument lets us choose $j$ so that the algorithm rarely returns that same candidate partition. The output therefore reliably reveals whether $D=Q$ or $D=P$.

Non-adaptivity lets us control the $Q$-budget as well. Since the allocation is fixed before observing any samples, the copies of $Q$ at the least-sampled indices have total expected sample count at most $\frac{T(r-1)}{k}$ under either LFHT hypothesis. Capping their sample count at a constant multiple of this bound gives a $Q$-budget of $O(\frac{Tr}{k})$, with only a small probability of aborting. Together with the lightly sampled index $j$, this gives budgets $(T,O(\frac{Tr}{k}),O(\frac{T}{k}))$, whereas the reduction of \citet{KPS26} allows $T$ samples from each reference. The smaller $Q$-budget lets us apply the asymmetric lower bounds of \cref{sec:lb}.

The formal reduction is stated as follows.

\begin{lemma}\label{lem:na-reduction}
Let $k\geq200$ and $1\leq r\leq k/2$, and let $\mathcal{A}$ be a non-adaptive algorithm for 2-clustering, with one known distribution or with both unknown. Suppose its (possibly random) allocation $(T_1,\ldots,T_k)$ satisfies $\sum_i T_i\leq T$, and it returns the correct size-$(r,k-r)$ partition with probability at least $\frac{99}{100}$ on every valid instance. Let $\Lambda$ be a distribution over pairs $(P,Q)$ with $\tv(P,Q)\geq\varepsilon$, where $P$ is fixed and known in the first variant. Then there exists an algorithm $\mathcal{A}'$ that decides whether $D=P$ or $D=Q$ using
\begin{align*}
s_P = T, \qquad
s_Q = \min\left(\left\lceil \frac{40\,Tr}{k} \right\rceil,\, T\right), \qquad
s_D = \min\left(\left\lceil \frac{160\,T}{k} \right\rceil,\, T\right)
\end{align*}
samples from $P$, $Q$, and $D$, respectively. Under each hypothesis, its error averaged over $(P,Q)\sim\Lambda$ is at most $\frac{1}{10}$. Moreover, $s_D\leq4s_Q$ and $T\geq k/5$, and hence
\begin{align}\label{eq:na-reduction-budgets}
    s_P=T,\qquad s_Q=O(Tr/k),\qquad s_D=O(T/k).
\end{align}
\end{lemma}

\begin{proof}
Let $\hat{\mathcal{I}}$ be the output of $\mathcal{A}$.
Relabel the indices so that $\mathbb{E}[T_1]\leq\cdots\leq\mathbb{E}[T_k]$.
The allocation is non-adaptive, so its distribution does not depend on the
instance.

\paragraph{Choice of $j$.}
For each $j \in \{r, \ldots, k\}$, let $S_j := \{1, \ldots, r-1\} \cup \{j\}$ and $\Pi_j:=\{S_j,[k]\setminus S_j\}$. These candidate partitions are distinct, including when $r=k/2$. Consider the random instance $(P,Q) \sim \Lambda$, $D_i = Q$ for $i < r$, $D_i = P$ for $i \geq r$. This comparison instance has only $r-1$ copies of $Q$ and need not satisfy the clustering promise. We use only the distribution of the algorithm's output on it, without assuming correctness. We show that there exists $j \in \{r, \ldots, k\}$ such that:
\begin{itemize}
	\item $\Pr[\hat{\mathcal{I}} = \Pi_j] \leq \frac{4}{k}$, and
	\item $\mathbb{E}[T_j] \leq \frac{4T}{k}$.
\end{itemize}
The events $\{\hat{\mathcal{I}} = \Pi_j\}_{j \geq r}$ are mutually exclusive, so
$\sum_{j \geq r} \Pr[\hat{\mathcal{I}} = \Pi_j] \leq 1$, and by averaging, more than half
of the $j \geq r$ satisfy $\Pr[\hat{\mathcal{I}} = \Pi_j] \leq \frac{2}{k-r+1} \leq \frac{4}{k}$
(using $r \leq \frac{k}{2}$). Likewise $\sum_{j \geq r} \mathbb{E}[T_j] \leq T$, so more than
half satisfy $\mathbb{E}[T_j] \leq \frac{4T}{k}$. Fix any $j$ satisfying both.

\paragraph{Algorithm $\mathcal{A}'$.} $\mathcal{A}'$ requests $s_P$, $s_Q$, and $s_D$ samples from the $P$-, $Q$-, and $D$-oracles, and runs $\mathcal{A}$ on the simulated instance
\begin{align*}
D_i := Q \ \ (i < r), \qquad D_j := D, \qquad D_i := P \ \ (i \notin S_j),
\end{align*}
forwarding its sample requests to the requested samples. If $\mathcal{A}$ requests more than $s_Q$ samples in total from $\{D_i\}_{i < r}$, or more than $s_D$ samples from $D_j$, then $\mathcal{A}'$ aborts and returns $D=P$. Otherwise it returns $D=Q$ if $\hat{\mathcal{I}} = \Pi_j$, and $D=P$ otherwise. The requests to the $P$-oracle never exceed $s_P=T$.

Recall that $s_Q$ and $s_D$ take the form $\min(\dotsc,T)$.  Whenever $T$ achieves the minimum, an abort is impossible. Otherwise, the two respective abort probabilities are bounded via Markov's inequality:
\begin{align*}
\Pr\Big[\sum_{i < r} T_i > s_Q\Big] \leq \frac{\sum_{i < r} \mathbb{E}[T_i]}{\frac{40\,Tr}{k}} \leq \frac{\frac{T(r-1)}{k}}{\frac{40\,Tr}{k}} \leq \frac{1}{40},
\qquad
\Pr[T_j > s_D] \leq \frac{\mathbb{E}[T_j]}{\frac{160\,T}{k}} \leq \frac{\frac{4T}{k}}{\frac{160\,T}{k}} = \frac{1}{40},
\end{align*}
using $\sum_{i < r} \mathbb{E}[T_i] \leq \frac{T(r-1)}{k}$ (the smallest $r-1$ of the sorted expectations sum to at most that much, since $\sum_i \mathbb{E}[T_i] \leq T$) and the choice of $j$.

\paragraph{Case $D = P$.} An abort returns $D=P$, which is correct, and otherwise the simulated instance is the one used to pick $j$, so the error probability is at most
$\Pr[\hat{\mathcal{I}} = \Pi_j] \leq \frac{4}{k} \leq \frac{1}{50}$ (for $k \geq 200$).

\paragraph{Case $D = Q$.} The instance is valid for $\mathcal{A}$ with true
partition $\Pi_j$, so by its success guarantee and the two abort bounds, the error is at most $\frac{1}{100} + \frac{2}{40} < \frac{1}{10}$.

Finally, $\left\lceil \frac{160\,T}{k} \right\rceil \leq 4\left\lceil \frac{40\,Tr}{k} \right\rceil$ since $r\geq1$, and $s_D\leq T$, so $s_D\leq4s_Q$ in each case of the two minima.

To control rounding in the budgets, couple the two hypotheses using the same pair $(P,Q)$, allocation, internal randomness, and reference samples. If $T_j=0$, the simulated observations and the tester's answer coincide under both hypotheses, so at least one of the two answers is wrong. Since each error probability is at most $1/10$, we must have $\Pr[T_j\geq1]\geq4/5$. Hence $4T/k\geq\E[T_j]\geq4/5$, giving $T\geq k/5$. The ceilings in $s_Q,s_D$ can therefore be absorbed into the bounds in~\eqref{eq:na-reduction-budgets}.
\end{proof}

\subsubsection{Application with One Known Distribution}\label{app:clustering-lower-known}

Here we take the large-cluster distribution to be $U$ and use the following existing average-case LFHT lower bound.

\begin{lemma}\label{lem:lfht-pu}
\emph{(\citet[Proposition~13]{GP24})} 
For $\varepsilon\in(0,1/8]$, there exists a distribution $\Lambda_{\mathrm{alt}}$ over pairs $(U,Q)$ with $\tv(U,Q)\geq\varepsilon$ satisfying the following. Any tester deciding $D=Q$ vs.\ $D=U$ with average error at most $1/10$ under each hypothesis, using $s_Q$ samples from $Q$ and $s_D$ from $D$ with $s_Q\gtrsim s_D$, must satisfy
\begin{align*}
    s_Q\gtrsim\frac{\sqrt n}{\varepsilon^2},\qquad
    s_Qs_D\gtrsim\frac{n}{\varepsilon^4},\qquad
    s_D\gtrsim\frac1{\varepsilon^2}.
\end{align*}
\end{lemma}

\begin{restatable}{theorem}{thmnalbcone}\label{thm:na-lb-c1}
For every $r \in \left[1, \frac{k}{2}\right]$, $k \geq 200$, and $\varepsilon\in(0,1/8]$, with one distribution known,
\begin{align*}
T_{\mathrm{na}}(n, k, r, \varepsilon) \;=\; \Omega\!\left(\frac{k\sqrt n}{\sqrt r\,\varepsilon^2} + \frac{k}{\varepsilon^2}\right).
\end{align*}
\end{restatable}
\begin{proof}
Apply \cref{lem:na-reduction} with $P=U$ and $(U,Q)\sim\Lambda_{\mathrm{alt}}$, and generate the $U$-samples directly. Since $s_D\leq4s_Q$, \cref{lem:lfht-pu} applies. The two required bounds follow by substituting~\eqref{eq:na-reduction-budgets}:
\begin{equation*}
    \frac{n}{\varepsilon^4}\lesssim s_Qs_D
    \lesssim\frac{T^2r}{k^2},\qquad 
    \frac1{\varepsilon^2}\lesssim s_D\lesssim\frac Tk.
\end{equation*}
\end{proof}

Together with \cref{thm:subset-c1}, this proves the characterization with one known distribution in \cref{sec:clustering}.

\subsubsection{Application with Both Distributions Unknown}\label{app:clustering-lower-unknown}

The known-distribution lower bound still applies when both distributions are unknown. To obtain the additional term, we apply the same reduction to the hard pairs from \cref{sec:lb-gap}.

\begin{restatable}{theorem}{thmnalbctwo}\label{thm:na-lb-c2}
For every $r \in \left[1, \frac{k}{2}\right]$, $k \geq 200$, and $\varepsilon\in(0,1/8]$, with both distributions unknown,
\begin{align*}
T_{\mathrm{na}}(n, k, r, \varepsilon) \;=\; \Omega\!\left(\frac{k\sqrt n}{\sqrt r\,\varepsilon^2} \;+\; \frac{k}{\varepsilon^2}\right),
\end{align*}
and, when additionally $n \gtrsim \frac{k^2}{r\varepsilon^4}$,
\begin{align*}
T_{\mathrm{na}}(n, k, r, \varepsilon) \;=\; \Omega\!\left(\frac{k^{2/3}\, n^{2/3}}{r^{1/3}\,\varepsilon^{4/3}}\right).
\end{align*}
\end{restatable}
\begin{proof}
The first bound is \cref{thm:na-lb-c1}, since a clustering algorithm for two unknown distributions can also be used when one is known. For the additional bound, suppose $n\gtrsim k^2/(r\varepsilon^4)$ and set
\begin{align*}
    T^*:=\frac{k^{2/3}n^{2/3}}{r^{1/3}\varepsilon^{4/3}}.
\end{align*}
This regime gives $T^*\lesssim n$, so $T>n/100$ already implies $T\gtrsim T^*$. We henceforth assume that $T\leq n/100$.

Apply \cref{lem:na-reduction} with $\Lambda$ the common distribution of $(P,Q)$ under $\Lambda_0(\cdot\mid E)$ and $\Lambda_1(\cdot\mid E)$. The resulting tester has average error at most $1/10$ under each hypothesis and budgets $1\leq s_Q,s_D\leq s_P=T\leq n/100$. The average-case lower bound in \cref{rem:gap-avg} therefore applies; using~\eqref{eq:na-reduction-budgets} and $r\geq1$, it gives
\begin{align*}
    \frac{n^2}{\varepsilon^4}
    \lesssim s_Ps_D(s_Q+s_D)
    \lesssim T\cdot\frac Tk\cdot\frac{Tr}{k}
    =\frac{T^3r}{k^2}.
\end{align*}
Rearranging yields $T\gtrsim T^*$.
\end{proof}

Together with \cref{thm:subset-c2}, this proves the characterization with both distributions unknown in \cref{sec:clustering}.

\subsection{Proof of the Adaptivity Gap}\label{app:adaptivity-gap}

\citet[Table~1 and Appendix~F]{KPS26} give the adaptive complexity
\begin{align*}
    T_{\mathrm a}=\widetilde\Theta\!\left(
        \frac{k^{2/3}n^{2/3}}{r^{2/3}\varepsilon^{4/3}}
        +\frac{k\sqrt n}{r\varepsilon^2}
        +\frac{k^{1/3}n^{2/3}}{\varepsilon^{4/3}}
        +\frac{\sqrt{nk}}{\varepsilon^2}
        +\frac{k}{\varepsilon^2}\right).
\end{align*}
The term $k/\varepsilon^2$ is shared. For the other terms, canceling the
common factors in $n$ and $\varepsilon$ gives the ratios
\begin{align*}
    \frac{k^{2/3}/r^{1/3}}
         {\max\{k^{2/3}/r^{2/3},\,k^{1/3}\}}
    &=\min\{r^{1/3},\,(k/r)^{1/3}\}\leq k^{1/6},\\
    \frac{k/\sqrt r}{\max\{k/r,\,\sqrt k\}}
    &=\min\{\sqrt r,\,\sqrt{k/r}\}\leq k^{1/4}.
\end{align*}
Hence, $T_{\mathrm{na}}/T_{\mathrm a}=\widetilde O(k^{1/4})$.
The example in \cref{sec:clustering} attains this bound up to logarithmic factors.

\section{Extension to Multiple Hypotheses}\label{sec:d-ary}

The problem extends naturally to $d\geq 2$ candidates. Let $P_1,\ldots,P_d$ be distributions on $[n]$, and let $D=P_{j^*}$ for an unknown $j^*\in[d]$ with $\tv(D,P_i)\geq\varepsilon$ for every $i\neq j^*$. Given $s_i$ samples from each $P_i$ and $s_D$ samples from $D$, the task is to output $j^*$. Let $s_{(1)}\leq s_{(2)}$ be the two smallest of $s_1,\ldots,s_d$.

\begin{restatable}{proposition}{propdary}\label{prop:d-ary}
	Let $n\geq 11$, $\varepsilon\in(0,1/8]$ and $2\leq d\leq n$.
	The set $\RLF^{(d)}(\varepsilon)$ of feasible sample budgets satisfies
	\begin{gather*}
		\RLF^{(d)}(\varepsilon)
		\asymp_{\log n}
		\left\{
		(s_1,\ldots,s_d,s_D):
		s_D \geq \frac{1}{\varepsilon^2},\quad
		s_{(2)} \geq \frac{\sqrt{n\alpha}}{\varepsilon^2},\quad
		s_D\bigl(s_{(1)}+s_D\bigr) \geq \frac{n\alpha}{\varepsilon^4}
		\right\},\\
		\text{where }
		\alpha=\max\Bigl\{1,\frac{n}{\max\{s_{(2)},s_D\}}\Bigr\}.
	\end{gather*}
\end{restatable}

We observe that all but one candidate must be sampled at the rate $\frac{\sqrt{n\alpha}}{\varepsilon^2}$, resembling identity testing, and only the least-sampled candidate trades against $s_D$ in the way $Q$ does in \cref{thm:char}. Up to the logarithmic factors in \cref{prop:d-ary}, the least-sampled candidate needs no samples when $s_D\gtrsim\frac{\sqrt{n\alpha}}{\varepsilon^2}$ and is identified by elimination. The upper bound loses a factor $O(\log d)$ in each sample size. Since $d\leq n$, this factor is hidden by $\asymp_{\log n}$.

\begin{proof}
    For the upper bound, we separately apply \cref{thm:char} to every pair $(P_i,P_j)$, with suitable amplification described below, and output the candidate that wins all of its comparisons. (If no such candidate exists, an error is declared.) The amplification modification is that we use $O(\log d)$ independent repetitions to reduce the error probability from $1/3$ to at most $\frac{1}{3d}$, with the repetitions reused across all pairwise comparisons.  Candidate $j^*$ satisfies the promise of every comparison it takes part in, so by a union bound over these $d-1$ comparisons it wins all of them, and every other candidate loses its comparison against $j^*$. Comparisons between two wrong candidates are inconsequential, so they need not be included in the union bound. Since $\RLF(\varepsilon)$ is closed under increasing any sample size, every pair is feasible once the pair with sizes $(s_{(2)},s_{(1)})$ is, and the amplification costs a factor $O(\log d)$ in each sample size.

	For the lower bound, we reduce from the pairwise problem with sizes $(s_{(2)},s_{(1)},s_D)$ on the domain $[n-1]$. Given such an instance $(P,Q,D)$, place $P$ in the slot with $s_{(2)}$ samples and $Q$ in the slot with $s_{(1)}$ samples, and fill each of the remaining $d-2$ slots with the point mass $\delta_n$, whose samples we can generate ourselves. Since $D$ is supported on $[n-1]$, we have $\tv(D,\delta_n)=1$. Together with $\tv(P,Q)\geq\varepsilon$, this makes the result a valid instance whose correct answer is the slot of $D$. The $d$-candidate tester therefore decides whether $D=P$ or $D=Q$, which gives a pairwise tester on $[n-1]$ with sizes $(s_{(2)},s_{(1)},s_D)$. The lower-bound direction of \cref{thm:char} on $[n-1]$ gives the three conditions with $n-1$ in place of $n$, and since $n-1\geq\frac{n}{2}$, this changes them only by constant factors.
\end{proof}

\end{document}